\documentclass{article}

\PassOptionsToPackage{numbers, sort&compress}{natbib}

\usepackage[preprint]{neurips_2024}

\usepackage[subrefformat=parens]{subcaption}

\usepackage{multirow} 
\usepackage{array} 
\usepackage{tabularx} 
\usepackage{ragged2e} 

\usepackage{graphicx}       
\usepackage{epstopdf}       
\usepackage[utf8]{inputenc} 
\usepackage[T1]{fontenc}    
\usepackage{hyperref}       
\usepackage{url}            
\usepackage{booktabs}       
\usepackage{amsfonts}       
\usepackage{nicefrac}       
\usepackage{microtype}      
\usepackage{xcolor}         

\usepackage{amsmath} 
\usepackage{amssymb} 
\usepackage{mathtools} 
\usepackage{amsthm} 
\usepackage{mathrsfs} 
\usepackage{DotArrow} 

\usepackage{algorithm,multicol} 

\usepackage{mystyles/customStyles}

\title{$\alpha$-Divergence Loss Function for Neural Density Ratio Estimation}

\author{%
  Yoshiaki Kitazawa \\
  NTT DATA Mathematical Systems Inc.\\
  Data Mining Division\\
  1F Shinanomachi Rengakan, 35, Shinanomachi\\
  Shinjuku-ku, Tokyo, 160-0016, Japan\\
  \texttt{kitazawa@msi.co.jp}
}

\begin{document} 

\maketitle 

\begin{abstract} 
   Density ratio estimation (DRE) is a fundamental machine learning technique for capturing relationships between two probability distributions. State-of-the-art DRE methods estimate the density ratio using neural networks trained with loss functions derived from variational representations of $f$-divergences. 
   However, existing methods face optimization challenges, such as overfitting due to lower-unbounded loss functions, biased mini-batch gradients, vanishing training loss gradients, and high sample requirements for Kullback--Leibler (KL) divergence loss functions. 
   To address these issues, we focus on $\alpha$-divergence, which provides a suitable variational representation of $f$-divergence. 
   Subsequently, a novel loss function for DRE, the $\alpha$-divergence loss function ($\alpha$-Div), is derived. 
      $\alpha$-Div is concise but offers stable and effective optimization for DRE. 
   The boundedness of $\alpha$-divergence provides the potential for successful DRE with data exhibiting high KL-divergence. 
      Our numerical experiments demonstrate the effectiveness of $\alpha$-Div in optimization.  
   However, the experiments also show that the proposed loss function offers no significant advantage over the KL-divergence loss function in terms of RMSE for DRE. This indicates that the accuracy of DRE is 
 primarily determined by the amount of KL-divergence in the data and is less dependent on $\alpha$-divergence. 
\end{abstract}

\section{Introduction} 
Density ratio estimation (DRE), a fundamental technique in various machine learning domains, estimates the density ratio $r^*(\mathbf{x})=q(\mathbf{x})/p(\mathbf{x})$ between two probability densities using two sample sets drawn separately from $p$ and $q$. Several machine learning methods, including generative modeling \citep{goodfellow2014generative,nowozin2016f,uehara2016generative}, mutual information estimation and representation learning \citep{belghazi2018mutual,hjelm2018learning}, energy-based modeling \citep{gutmann2010noise}, and covariate shift and domain adaptation \citep{shimodaira2000improving,huang2006correcting}, involve problems where DRE is applicable. Given its potential to enhance a wide range of machine learning methods, the development of effective DRE techniques has garnered significant attention. 

Recently, neural network-based methods for DRE have achieved state-of-the-art results. These methods train neural networks as density ratio functions using loss functions derived from variational representations of $f$-divergences \citep{nguyen2010estimating}, which are equivalent to density-ratio matching under Bregman divergence \citep{sugiyama2012density}. The optimal function for a variational representation of $f$-divergence, through the Legendre transform, corresponds to the density ratio. 

However, existing neural network methods suffer from several issues. First, an overfitting phenomenon, termed \textit{train-loss hacking} by \citet{kato2021non}, occurs during optimization when lower-unbounded loss functions are used. 
Second, the gradients of loss functions over mini-batch samples provide biased estimates of the full gradient when using standard loss functions derived directly from the variational representation of $f$-divergence \citep{belghazi2018mutual}. Third, loss function gradients can vanish when the estimated probability ratios approach zero or infinity \citep{arjovsky2017towards}. Finally, optimization with a Kullback--Leibler (KL) divergence loss function often fails on high KL-divergence data because the sample requirement for optimization increases exponentially with the true amount of KL-divergence \citep{poole2019variational,song2019understanding,mcallester2020formal}. 

To address these problems, this study focuses on $\alpha$-divergence, a subgroup of $f$-divergences, which has a sample complexity independent of its ground truth value. We then present a Gibbs density representation for a variational form of the divergence to obtain unbiased mini-batch gradients, from which we derive a novel loss function for DRE, referred to as the $\alpha$-divergence loss function ($\alpha$-Div). Despite its simplicity, $\alpha$-Div offers stable and effective optimization for DRE. 

Furthermore, this study provides technical justifications for the proposed loss function. $\alpha$-Div has a sample complexity that is independent of the ground truth value of $\alpha$-divergence and provides unbiased mini-batch gradients of training losses. Additionally, choosing $\alpha$ within the interval $(0,1]$ ensures that $\alpha$-Div remains lower-bounded, preventing train-loss hacking during optimization. By selecting $\alpha$ from this interval, we also avoid vanishing gradients in neural networks when they reach extreme local minima. We empirically validate our approach through numerical experiments using toy datasets, which demonstrate the stability and efficiency of the proposed loss function during optimization. 

However, we observe that the root mean squared error (RMSE) of the estimated density ratios increases significantly for data with higher KL-divergence when using the proposed loss function. The same phenomenon is observed with the KL-divergence loss function. 
These results suggest that the accuracy of DRE is primarily determined by the magnitude of KL-divergence inherent in the data, rather than by the specific choice of $\alpha$ in the $\alpha$-divergence loss function. This observation highlights a fundamental limitation shared across different $f$-divergence-based DRE methods, emphasizing the need to consider the intrinsic KL-divergence of data when evaluating and interpreting DRE performance.

The key contributions of this study are as follows: First, we propose a novel loss function for DRE, termed $\alpha$-Div, providing a concise solution to instability and biased gradient issues present in existing $f$-divergence-based loss functions. Second, technical justifications and theoretical insights supporting the proposed $\alpha$-Div loss function are presented. Third, we empirically confirm the stability and efficiency of the proposed method through numerical experiments. Finally, our empirical results reveal that the accuracy of DRE, measured by RMSE, is primarily influenced by the magnitude of KL-divergence inherent in the data, rather than by the specific choice of $\alpha$ in the $\alpha$-divergence loss function.

\section{Problem Setup}\label{section_problem_setup} 
\textbf{Problem definition.} $P$ and $Q$ are probability distributions on  $\Omega \subset \mathbb{R}^{d}$ with unknown probability densities $p$ and $q$, respectively. 
We assume $p(\mathbf{x}) > 0 \Leftrightarrow q(\mathbf{x}) > 0$ almost everywhere $\mathbf{x} \in \Omega$, ensuring that the density ratio is well-defined on their common support.

The goal of DRE is to accurately estimate $r^*(\mathbf{x})=q(\mathbf{x})/p(\mathbf{x})$ from given i.i.d. samples $\hat{\mathbf{X}}_{P[R]}=\{\mathbf{x}_i^p \}_{i=1}^R \sim p$ and $\hat{\mathbf{X}}_{Q[S]}=\{\mathbf{x}_i^q \}_{i=1}^S \sim q$. 

\textbf{Additional notation.} 
$E_P[\cdot]$ denotes the expectation under the distribution $P$: 
$E_P[\phi(\mathbf{x})]=\int_{\Omega} \phi(\mathbf{x})dP(\mathbf{x})$,  where $\phi(\mathbf{x})$ is a measurable function over $\Omega$. 
$\hat{E}_{P[R]}[\cdot]$ denotes the empirical expectation of $\hat{\mathbf{X}}_{P[R]}$: 
$\hat{E}_{P[R]}[\phi(\mathbf{x})]=\sum_{i=1}^{R} \phi(\mathbf{x}_i^p)/R$. 
The variables in a function or the superscript variable  ``$[R]$'' of $\hat{E}_{P[R]}$ can be omitted when unnecessary and instead represented as $E_P[\phi]$ or $\hat{E}_{P}[\phi]$. 
$I(\text{cond})$ denotes the indicator function: $I(\text{cond})=1$ if ``cond'' is true, and 0 otherwise.
Similarly, notations $E_Q[\cdot]$ and $\hat{E}_{Q[S]}[\cdot]$ are defined. 
$E[\cdot]$ is written for $E_P[E_Q[\cdot]]$.

\section{DRE via \texorpdfstring{$f$}{f}-divergence variational representations and its major problems}\label{section_problems_for_dre} 
In this section, we introduce DRE using $f$-divergence variational representations and $f$-divergence loss functions. First, we review the definition of $f$-divergences. Next, we identify four major issues with existing $f$-divergence loss functions: the overfitting problem with lower-unbounded loss functions, biased mini-batch gradients, vanishing training loss gradients, and high sample requirements for Kullback--Leibler (KL) divergence loss functions. 

\subsection{DRE via \texorpdfstring{$f$}{f}-divergence variational representation} 
First, we review the definition of $f$-divergences. 

\begin{definition}[$f$-divergence]\label{Definition_f_divergence} 
    The $f$-divergence $D_{f}$ between two probability measures $P$ and $Q$, which is induced by a convex function $f$ satisfying $f(1) = 0$, is defined as 
    $D_{f}(Q||P)=E_{P}[f(q(\mathbf{x})/p(\mathbf{x}))]$. 
\end{definition} 
Many divergences are specific cases obtained by selecting a suitable generator function $f$. 
For example, $f(u) = u \cdot \log u$ corresponds to KL-divergence. 

Then, we derive the variational representations of $f$-divergences 
using the Legendre transform of the convex conjugate of a twice differentiable convex function $f$, 
$f^*(\psi) = \sup_{u\in \mathbb{R}}\{\psi \cdot u - f(u)\}$ \citep{nguyen2007estimating}: 
\begin{equation} 
    D_{f} (Q||P) = \sup_{\phi \ge 0}\Big\{ E_Q\big[f'(\phi)\big] - E_P\big[f^*(f'(\phi)) \big] \Big\}, \label{Eq_Variational_representation_f_div_phi} 
\end{equation} 
where the supremum is taken over all measurable functions $\phi:\Omega \rightarrow \mathbb{R}$ 
with $ E_Q[\,|f'(\phi)|\,] < \infty$ and $E_P[\,|f^* (f'(\phi))|\,] < \infty$. 
The maximum value is achieved at $\phi(\mathbf{x})=q(\mathbf{x})/p(\mathbf{x})$. 

By replacing $\phi$ with a neural network model $\phi_{\theta}$, 
the optimal function for Equation  (\ref{Eq_Variational_representation_f_div_phi}) is trained 
through back-propagation using an $f$-divergence loss function, such that 
\begin{equation} 
    \mathcal{L}_f^{(R,S)}(\phi_{\theta}) =  - \left\{ \hat{E}_{Q[S]} \big[f'(\phi_{\theta})\big]- 
    \hat{E}_{P[R]}\big[f^*(f'(\phi_{\theta}))\big] \right\}, 
    \label{Eq_loss_func_f_div} 
\end{equation} 
where $\phi_{\theta}$ is a real-valued function, the superscript variable  ``$(R,S)$'' can be omitted when unnecessary and instead represented as $\mathcal{L}_f(\cdot)$. 
As shown in Table \ref{Table_for_loss_functions_for_DRE}, we list pairs of convex functions and the corresponding loss functions $\mathcal{L}_f(\phi_{\theta})$ in Equation (\ref{Eq_loss_func_f_div}) for several $f$-divergences.

\subsection{Train-loss hacking problem}\label{subsection_TrainlossHackingproblem} 
When $f$-divergence loss functions $\mathcal{L}_f(\phi_{\theta})$, as defined in Equation (\ref{Eq_loss_func_f_div}), are not lower-bounded, overfitting can occur during optimization. 
For example, we observe the case of the Pearson $\chi^2$ loss function, $\mathcal{L}_{\mathrm{chi}\text{-}\mathrm{sq}} = - 2 \cdot \hat{E}_Q\big[\phi_{\theta}\big] + \hat{E}_P\big[\phi_{\theta}^2\big]$, as follows. 
Since the term $- 2 \cdot \hat{E}_Q\big[\phi_{\theta}\big]$ is not lower-bounded, it can approach negative infinity, causing the entire loss function to diverge to negative infinity as $\phi_{\theta}(\mathbf{x}_i^q) \rightarrow \infty$ for $\mathbf{x}_i^q \in \hat{\mathbf{X}}_{Q[S]}$. 
Consequently, $\mathcal{L}_{\mathrm{chi}\text{-}\mathrm{sq}} \rightarrow - \infty$ when $\phi_{\theta}(\mathbf{x}_i^q) \rightarrow \infty$ for some $\mathbf{x}_i^q \in \hat{\mathbf{X}}_{Q[S]}$. 
As shown in Table \ref{Table_for_loss_functions_for_DRE}, both the KL-divergence and Pearson $\chi^2$ loss functions are not lower-bounded, and hence, are prone to overfitting during optimization. 
This phenomenon is referred to as \textit{train-loss hacking} by \citet{kato2021non}.


\subsection{Biased gradient problem}\label{subsection_Biasedgradientsproblem} 
Neural network parameters are updated using the accumulated gradients from each mini-batch.
It is desirable for these gradients to be unbiased, i.e., $E\big[\nabla_{\theta}\mathcal{L}_{f}(\theta)\big] = \nabla_{\theta} E\big[\mathcal{L}_f(\theta)\big]$ holds. 
However, the equality between $E\big[\nabla_{\theta}\mathcal{L}_{f}(\theta)\big]$ and $\nabla_{\theta} E\big[\mathcal{L}_f(\theta)\big]$ requires the uniform integrability of $\mathcal{L}_f(\theta)$, i.e.,  $\lim_{K\rightarrow\infty} \sup_{\theta} E\big[| \mathcal{L}_f(\theta)| \cdot I(\mathcal{L}_f(\theta) > K) \big] = 0$.
The uniform integrability condition is typically violated when the loss function exhibits heavy-tailed behavior, which often occurs for the standard $f$-divergence loss functions derived solely from Equation (\ref{Eq_loss_func_f_div}). 
Consequently, the standard loss functions frequently result in biased gradients. 

To illustrate this, consider employing the KL-divergence loss function for optimizing a shift parameter of $\phi_{\theta} =  |x - \theta| $, where $\theta \in (0,1)$ and $x \in [0,1]$.
Intuitively, in the above example, biased gradients occur because the KL-divergence loss gradients contain terms inversely proportional to the estimated density ratios, making their expectation diverge.
Specifically, the loss function is obtained as $\mathcal{L}_{KL}(\phi_{\theta}) = - \hat{E}_Q\big[\log \phi_{\theta}\big] + \hat{E}_P\big[\phi_{\theta}\big] - 1$, and the gradient is expressed as $\nabla_{\theta} \mathcal{L}_{KL}(\phi_{\theta}) = -\hat{E}_Q\big[ \nabla_{\theta}(\log \phi_{\theta})\big] + \hat{E}_P\big[\nabla_{\theta}(\phi_{\theta})\big]$. 
Then, we have $\frac{\partial}{\partial \theta } E\big[\log \phi_{\theta}(x)\big] = \frac{\partial}{\partial \theta } \int_{0}^{1} \log |x - \theta| \, dx = 
- \log(1 - \theta)$ and $E\big[\frac{\partial}{\partial \theta } \log  \phi_{\theta}(x)   \big] =  \int_{0}^{\theta} \frac{1}{\theta-x} dx   + \int_{\theta}^{1} \frac{1}{x-\theta} dx = \infty$. 
Consequently, we generally observe that $\nabla_{\theta}E\big[\mathcal{L}_{KL}(\phi_{\theta})\big] \neq E\big[\nabla_{\theta} \mathcal{L}_{KL}(\phi_{\theta})\big]$. 

To mitigate this issue, \citet{belghazi2018mutual} introduced a bias-reduction method for stochastic gradients in KL-divergence loss functions.

\subsection{Vanishing gradient problem}\label{subsection_Vanishinggradientsproblem} 
The vanishing gradient problem is a well-known issue in optimizing GANs \citep{arjovsky2017towards}. 
We suggest that this problem occurs when the following two conditions are met: (i) the loss function results in minimal updates of model parameters, and (ii) updating the model parameters leads to negligible changes in the model's outputs. Thus, the problem emerges when the following equation holds:

\begin{equation} 
    \underbrace{E\big[\nabla_{\theta} \mathcal{L}_f(\phi_{\theta})\big] = \mathbf{0}}_{\text{(i)}}   \  \  \text{\&} \  \ 
    \underbrace{E_Q\big[\nabla_{\theta} \phi_{\theta}\big] = \mathbf{0} \  \ \text{\&} \  \  E_P\big[\nabla_{\theta} \phi_{\theta}\big]  = \mathbf{0}}_{\text{(ii)}}, 
    \label{Eq_gradient_vanishing} 
\end{equation} 
where $\mathbf{0}$ denotes a zero vector of the same dimension as the model gradient.

In Equation (\ref{Eq_gradient_vanishing}), condition (i) describes the loss function's gradient vanishing, while condition (ii) ensures that the vanishing gradient condition persists. 
Specifically, the following three observations clarify their relationship:
First, condition (i) does not necessarily imply condition (ii);
Second, condition (i) alone does not guarantee its own persistence;
Third, condition (ii) ensures the continued validity of condition (i).

To illustrate the first observation, consider the KL-divergence loss. 
Because its gradient is obtained as $- \hat{E}_Q \big[ \nabla_{\theta} \phi_{\theta} / \phi_{\theta} \big] + \hat{E}_P\big[\nabla_{\theta}\phi_{\theta}\big]$ as shown in Table \ref{Table_about_gradient_vanishing}, (i) reduces to $E_Q\big[ \nabla_{\theta} \phi_{\theta} / \phi_{\theta} \big] = E_P\big[\nabla_{\theta}\phi_{\theta}\big]$. 
Then, (i) does not ensure (ii). 
To understand the second and third observations, note that (ii) is both necessary and sufficient for preventing updates to the model parameters. 
Thus, if condition (ii) does not hold, the model's predictions may change, potentially leading to the breakdown of condition (i). 
On the other hand, if condition (ii) holds, the model updates do not alter the outputs, leaving the gradient condition (i) unchanged, thereby causing the vanishing gradient problem to persist.

Consider the scenario where estimated density ratios become extremely small or large, fulfilling sufficient conditions for Equation (\ref{Eq_gradient_vanishing}) to hold. 
Table \ref{Table_about_gradient_vanishing} presents the gradient formulas for the divergence loss functions (as provided in Table \ref{Table_for_loss_functions_for_DRE}) along with their asymptotic behavior of the loss gradients as $\phi_{\theta} \rightarrow 0$ or $\phi_{\theta} \rightarrow \infty$. 
These results demonstrate that major $f$-divergence loss functions satisfy the conditions for Equation (\ref{Eq_gradient_vanishing}), showing that $E \big[\nabla_{\theta} \mathcal{L}_f(\phi_{\theta})\big] \rightarrow c_1 \cdot E_Q\big[\nabla_{\theta} \phi_{\theta}\big] + c_2 \cdot E_P\big[\nabla_{\theta} \phi_{\theta}\big]$, where $c_1$ and $c_2$ are constants, as $\phi_{\theta} \rightarrow 0$ or $\phi_{\theta} \rightarrow \infty$. 
In summary, all the divergence loss functions in Tables \ref{Table_for_loss_functions_for_DRE} and \ref{Table_about_gradient_vanishing} 
can experience vanishing gradients when the estimated density ratio approaches extremal estimates. 

\subsection{Sample size requirement problem for KL-divergence}\label{subsection_SamplesizerequirementproblemforKLdivergence} 
The sample complexity of the KL-divergence is $O(e^{KL(Q||P)})$, which implies that 
\begin{equation} \lim_{N \rightarrow \infty} N \cdot \mathrm{Var} \Big[ \widehat{KL^N}(Q||P) \Big] \ge e^{KL(Q||P)} - 1, \label{Eq_SamplesizerequirementproblemforKLdivergence} \end{equation} where $\widehat{KL^N}(Q||P)$ represents an arbitrary KL-divergence estimator for a sample size $N$ using a variational representation of the divergence, 
and $KL(Q||P)$ represents the true value of KL-divergence \citep{poole2019variational,song2019understanding,mcallester2020formal}. 
That is, when using KL-divergence loss functions, the sample size of the training data must increase exponentially as the true amount of KL-divergence increases in order to sufficiently train a neural network. 
To address this issue, existing methods divide the estimation of high divergence values into multiple smaller divergence estimations \citep{rhodes2020telescoping}.

\begin{table}[t] 
    \caption{
        List of $f$-divergence loss functions $\mathcal{L}_f(\phi_{\theta})$ in Equation (\ref{Eq_loss_func_f_div}), along with their associated convex functions and their lower-boundedness status.
        Part of the list of divergences and their convex functions is based on \citet{nowozin2016f}.} 
    \label{Table_for_loss_functions_for_DRE} 
    \centering 
     \vspace{3.0mm} 
    \begin{tabular}{lccc} 
        \toprule            
        Name 	     &convex function $f$  & $\mathcal{L}_f(\phi_{\theta})$ & \begin{tabular}{l} 
            Lower-\\ 
            bounded? 
        \end{tabular}\\ 
        \midrule             
        KL &   $u \cdot \log u$& 
        $- \hat{E}_Q\big[\log(\phi_{\theta})\big] +  \hat{E}_P\big[\phi_{\theta}\big] - 1$& No \\ 
        Pearson $\chi^2$     & $(u-1)^2$          & 
        $-2 \cdot \hat{E}_Q\big[\phi_{\theta}\big] +  \hat{E}_P\big[\phi_{\theta}^2\big] + 1$&  No \\ 
        Squared Hellinger   & $(\sqrt{u} - 1)^2$ &  $\hat{E}_Q\big[\phi_{\theta}^{-1/2}\big] 
         +  \hat{E}_P\big[\phi_{\theta}^{1/2}\big] - 2$ &    Yes                \\ 
        GAN      & \begin{tabular}{l}$ u\cdot\log u$ \\ 
                       $\ -(u+1)  \log (u+1)$ 
                    \end{tabular}  & 
        \begin{tabular}{l} 
            $\hat{E}_Q\big[ \log (1+\phi_{\theta}^{-1} )\big]$\\ 
            \ \   $+ \ \hat{E}_P\big[\log (1+\phi_{\theta})\big]$ \\ 
        \end{tabular} & Yes  \\ 
        \bottomrule          
    \end{tabular} 
     \vspace{3.0mm} 
    \caption{ 
        List of gradient formulas $\nabla_{\theta} \mathcal{L}_f(\phi_{\theta})$ of loss functions $\mathcal{L}_f(\phi_{\theta})$ in Table \ref{Table_for_loss_functions_for_DRE} and 
        the asymptotic behavior of $E\big[\nabla_{\theta} \mathcal{L}_f(\phi_{\theta})\big]$ as $\phi_{\theta} \rightarrow 0$ or $\phi_{\theta} \rightarrow \infty$ under regular conditions. 
    A symbol  ``*'' in the table 
    indicates that the asymptotic value cannot be expressed as a linear combination of $E_Q\big[\nabla_{\theta}\phi_{\theta}\big]$ and $E_P\big[ \nabla_{\theta}\phi_{\theta}\big]$, and 
        $\mathbf{0}$ denotes a vector of zeros with the same length as the model gradient. 
    } 
    \label{Table_about_gradient_vanishing} 
    \centering 
     \vspace{3.0mm} 
    \begin{tabular}{lccc} 
        \toprule            
        \multicolumn{2}{c}{} & 
        \multicolumn{2}{c}{$E\big[\nabla_{\theta} \mathcal{L}_f(\phi_{\theta})\big] \rightarrow  \text{?}$}\\ 
        \cmidrule(lr){3-4} 
        Name 	     & 
        $\nabla_{\theta} \mathcal{L}_f(\phi_{\theta})$ & 
        $\phi_{\theta} \rightarrow 0$& 
        $\phi_{\theta} \rightarrow \infty$\\ 
        \midrule             
        KL & 
        $- \hat{E}_Q\big[\nabla_{\theta}\phi_{\theta}/\phi_{\theta}\big] + \hat{E}_P\big[\nabla_{\theta}\phi_{\theta}\big]$& 
        * & 
        $ E_P\big[\nabla_{\theta} \phi_{\theta}\big] $\\ 
        Pearson $\chi^s$     & 
        $-2 \cdot \hat{E}_Q\big[\nabla_{\theta}\phi_{\theta}\big] 
         +  2 \cdot\hat{E}_P\big[ \nabla_{\theta}\phi_{\theta} \cdot \phi_{\theta}\big]$       & 
        $-2 \cdot E_Q\big[\nabla_{\theta}\phi_{\theta}\big]$ & *  \\ 
        Squared Hellinger   & 
        \begin{tabular}{l} 
            $-\frac{1}{2} \cdot \hat{E}_Q\big[\nabla_{\theta}\phi_{\theta}\cdot \phi_{\theta}^{-3/2}\big ]$\\ 
            $ \ \ \   + \  \frac{1}{2} \cdot \hat{E}_P\big[ \nabla_{\theta}\phi_{\theta} \cdot \phi_{\theta}^{-1/2}\big]$ 
        \end{tabular}  & 
        * &  $\mathbf{0}$\\ 
         GAN       & 
         \begin{tabular}{l} 
         $- \hat{E}_Q\big[ \nabla_{\theta}\phi_{\theta} / \big\{\phi_{\theta} \cdot (1+\phi_{\theta}) \big\}\big]$\\ 
         $\ \ \ \ \ \ \ + \  \hat{E}_P\big[\nabla_{\theta}\phi_{\theta}/(1+\phi_{\theta})\big]$ 
        \end{tabular} 
         & 
        \begin{tabular}{l} 
           * 
        \end{tabular} 
         & $\mathbf{0}$\\ 
        \bottomrule          
    \end{tabular} 
\end{table} 

\section{DRE using a neural network with an \texorpdfstring{$\alpha$}{α}-divergence loss}
\label{Section_DREusinganeuralnetworkwithanalphadivergenceloss} 
In this section, we derive our loss function from a variational representation of $\alpha$-divergence and present the training and prediction methods using this loss function. 
The exact claims and proofs for all theorems are deferred to Section \ref{Section_Appendix_proofs} in the Appendix. 

\subsection{Derivation of our loss function for DRE}\label{DerivationofourlossfunctionforDRE} 
Here, we define $\alpha$-divergence (Amari's $\alpha$-divergence), which is a subgroup of $f$-divergence, as \citep{amari2000methods}: 
\begin{equation} 
    D_{\alpha}(Q||P)=E_{P}\left[\frac{1}{\alpha\cdot(\alpha-1)} \cdot \left\{ \left( \frac{q(\mathbf{x})}{p(\mathbf{x})}\right)^{1 - \alpha} - 1 \right\}  \right], \label{Eq_alpha_div_def} 
\end{equation} 
where $\alpha \in \mathbb{R} \setminus \{0, 1\}$. 
From Equation (\ref{Eq_alpha_div_def}), Hellinger divergence is obtained when $\alpha=1/2$, and $\chi^2$ when $\alpha=-1$.

Then, we obtain the following variational representation of $\alpha$-divergence : 
\begin{theorem}\label{theorem_alpha_div_resp} 
    A variational representation of $\alpha$-divergence is given as 
    \begin{equation} 
        D_{\alpha} (Q||P) =  \sup_{\phi \ge 0} \left\{ 
        \frac{1}{\alpha \cdot (1-\alpha)} - \frac{1}{\alpha}\cdot  E_{Q} \Big[\phi^{\alpha} \Big] 
        - \frac{1}{1 - \alpha} \cdot E_{P} \Big[\phi^{\alpha - 1} \Big]  \right\}, \label{theorem_Eq_alpha_variation} 
    \end{equation} 
    where the supremum is taken over all measurable functions satisfying $ E_P[\phi^{1 - \alpha}] < \infty$ and $E_Q[\phi^{-\alpha}] < \infty$. 
    The maximum value is achieved at $\phi(\mathbf{x})=q(\mathbf{x})/p(\mathbf{x})$. 
\end{theorem} 

From the right-hand side of Equation (\ref{theorem_Eq_alpha_variation}), we obtain a standard $\alpha$-divergence loss function as 
\begin{equation} 
    \mathcal{L}_{\alpha\text{-standard}}^{(R,S)}(\phi_{\theta}\, ;\, \alpha) 
    =  \frac{1}{\alpha}\cdot  \hat{E}_{Q} \Big[\phi_{\theta}^{\alpha} \Big] 
       + \frac{1}{1 - \alpha} \cdot \hat{E}_{P} \Big[\phi_{\theta}^{\alpha - 1} \Big].  \label{Eq_loss_func_standard_alpha_div} 
\end{equation} 

Because $\nabla_{\theta} E_Q\big[ \phi_{\theta}^{\alpha}\big] \allowbreak \neq \allowbreak E_Q\big[ \nabla_{\theta}  (\phi_{\theta}^{\alpha})\big]$ 
and $\nabla_{\theta} E_P\big[ \phi_{\theta}^{\alpha - 1}\big] \allowbreak \neq \allowbreak E_P\big[\nabla_{\theta}  (\phi_{\theta}^{\alpha - 1})\big]$ are generally observed when $\alpha < 2$,  the standard $\alpha$-divergence loss function with $\alpha < 2$ has biased gradients. 

To obtain unbiased gradients for any $\alpha$, we rewrite the terms $ \phi_{\theta}^{\alpha}$ and $ \phi_{\theta}^{\alpha - 1}$ of the equation in Gibbs density form. 
Then, we have another variational representation of $\alpha$-divergence. 
\begin{theorem} \label{theorem_alpha_div_resp_in_gibbs_dinsity_form} 
    A variational representation of $\alpha$-divergence is given as 
    \begin{equation} 
        D_{\alpha} (Q||P)  =  \sup_{T:\Omega \rightarrow \mathbb{R}} \left\{ 
        \frac{1}{\alpha \cdot (1-\alpha)}  - \frac{1}{\alpha} \cdot E_{Q} \Big[e^{\alpha \cdot T} \Big] 
        - \frac{1}{1- \alpha} \cdot E_{P} \Big[ e^{(\alpha - 1) \cdot T} \Big] \right\}, 
        \label{theorem_alpha_div_loss_for_training} 
    \end{equation} 
    where the supremum is taken over all measurable function $T:\Omega \rightarrow \mathbb{R}$ satisfying 
    $E_P[e^{(\alpha - 1) \cdot T}] < \infty$ and $E_{Q} [e^{\alpha \cdot T}] < \infty$. 
    The equality holds for $T^*$ satisfying $e^{-T^*(\mathbf{x})} = q(\mathbf{x})/p(\mathbf{x})$. 
\end{theorem}

Subsequently, we obtain our loss function for DRE, called $\alpha$-Divergence loss function ($\alpha$-Div). 
\begin{definition}[$\alpha$-Div] 
    $\alpha$-Divergence loss is defined as: 
    \begin{equation} 
        \mathcal{L}_{\alpha\text{-Div}}^{(R,S)}(T_{\theta}\, ;\, \alpha) 
        = \frac{1}{\alpha}  \cdot  \hat{E}_{Q[S]} \Big[  e^{\alpha\cdot T_{\theta}} \Big] 
        + \frac{1}{1- \alpha} \cdot \hat{E}_{P[R]} \Big[e^{(\alpha -1 )\cdot T_{\theta}} \Big].  \label{Eq_loss_func_alpha_div} 
    \end{equation} 
\end{definition} 
The superscript  ``$(R,S)$'' is dropped when unnecessary and instead expressed as $\mathcal{L}_{\alpha\text{-Div}}(T_{\theta}\, ;\, \alpha)$.

\subsection{Training and predicting with \texorpdfstring{$\alpha$}{α}-Div} 
We train a neural network with $\alpha$-Div as described in Algorithm \ref{algo_train_DRE}. 
In practice, neural networks rarely achieve the global optimum in Equation  (\ref{theorem_alpha_div_loss_for_training}). 
The following theorem suggests that normalizing the estimated values, 
$q(\mathbf{x})/p(\mathbf{x}) = e^{-T_\theta(\mathbf{x})} / \hat{E}_P\left[e^{-T_\theta}\right]$, 
improves the optimization of the neural networks.

\begin{theorem} \label{theorem_alpha_div_loss_in_const_shift} 
    For a fixed function $T:\Omega \rightarrow \mathbb{R}$, 
    let $c^*$ be the optimal scalar value for the following infimum: 
    \begin{equation} 
        c^*  = \arg \inf_{c \in \mathbb{R}} E\Big[\mathcal{L}_{\alpha\text{-Div}}(T - c)\Big] 
        = \arg \inf_{c \in \mathbb{R}} \left\{ 
        \frac{1}{\alpha} \cdot  E_{Q} \Big[e^{\alpha \cdot(T - c)} \Big] 
        +  \frac{1}{1- \alpha} \cdot E_{P} \Big[ e^{(\alpha - 1) \cdot (T-c)} \Big] \right\}. 
        \label{Eq_theorem_alpha_div_loss_for_training_in_const_shift} 
    \end{equation} 
    Then, $c^*$ satisfies $E_{P} \left[ e^{- T - c^*} \right] = 1$. 
    That is, $e^{- T - c^*} = e^{- T} / E_{P} \left[ e^{- T} \right]$. 
\end{theorem}

\begin{algorithm}[t] 
    \caption{Training for DRE with $\alpha$-Div}\label{algo_train_DRE} 
    \begin{multicols}{2} 
        \begin{algorithmic} 
          \Require Data from denominator distribution $\{\mathbf{x}_i^p \}_{i=1}^R$, data from numerator distribution $\{\mathbf{x}_i^q \}_{i=1}^S$, learning rate $\eta$, and initial parameters $\theta_1$.
          \Ensure A neural network model $T_{\theta_N}$.
        \end{algorithmic} 
        \columnbreak 
        \begin{algorithmic}
          \For{$t=1$ to $N$} 
          \State $\hat{E}_{P} \leftarrow \frac{1}{R} \sum_{i=1}^R e^{(\alpha - 1) \cdot T_{\theta_t}(\mathbf{x}_i^p)}$ 
          \State $\hat{E}_{Q} \leftarrow \frac{1}{S} \sum_{i=1}^S e^{\alpha \cdot T_{\theta_t}(\mathbf{x}_i^q)}$ 
          \State $\mathcal{L}_{\alpha\text{-Div}}({\theta_t}) \leftarrow \hat{E}_{Q}/\alpha + \hat{E}_{P}/(1 - \alpha)$ 
          \State $\theta_{t+1} \leftarrow \theta_t  - \eta \cdot \nabla_{\theta_t} \mathcal{L}_{\alpha\text{-Div}}({\theta_t})$ 
          \EndFor
          \State \textbf{Return} trained parameters $\theta_{N+1}$.
      \end{algorithmic} 
    \end{multicols} 
    \vskip -0.1in 
\end{algorithm}

\section{Theoretical results for the proposed loss function}\label{Section_TheoreticaljustificationsofalphaDiv} 
In this section, we provide theoretical results that justify our approach with $\alpha$-Div. 
The exact claims and proofs for all the theorems are deferred to Section \ref{Appendix_Section_proof_TheoreticaljustificationsofalphaDiv} in the Appendix.

\subsection{Addressing the train-loss hacking problem}\label{Section_AddressingTrainLossHackingProblem} 

$\alpha$-Div avoids the train-loss hacking problem when $\alpha$ is within $(0, 1)$. 
Table \ref{Table_for_statusofthelowerboundedness_of_alphadiv_for_DRE} summarizes the lower-boundedness status of $\alpha$-Div for each case: $\alpha < 0$, $0 < \alpha < 1$, or $\alpha > 1$. 
$\alpha$-Div is lower-bounded when $0 < \alpha < 1$, whereas it is not lower-bounded when $\alpha > 1$ or $\alpha < 0$. 
Thus, selecting $\alpha$ from the interval (0, 1) effectively prevents the train-loss hacking problem.

\subsection{Unbiasedness of gradients}\label{Section_Unbiasednessofgradients}
By rewriting $\phi_{\theta}^{\alpha}$ in the Gibbs density form $e^{\alpha \cdot T_{\theta}}$, we mitigate the heavy-tailed behavior that often breaks uniform integrability in the standard $\alpha$-divergence loss functions. Hence, we present Theorem~\ref{theorem_gradient_unbiased}, which guarantees the unbiasedness of $\alpha$-Div's gradients. 

\begin{theorem}[Informal statement] \label{theorem_gradient_unbiased} 
    Let $T_{\theta}(\mathbf{x}):\Omega \rightarrow \mathbb{R}$ be a function such that 
    the map $\theta = (\theta_1,\theta_2, \ldots, \theta_p) \in \Theta \mapsto T_{\theta}(\mathbf{x})$ is differentiable for all $\theta$ and for $\mu$-almost every $\mathbf{x} \in \Omega$. 
    Under some regularity conditions, including the local Lipschitz continuity of $T_{\theta}$, we have
    \begin{equation} 
        E\Big[ \nabla_{\theta}\mathcal{L}_{\alpha\text{-Div}}(T_{\theta}; \alpha)\big|_{\theta = \bar{\theta}} \Big] 
        = \nabla_{\theta} E\Big[ \mathcal{L}_{\alpha\text{-Div}}(T_{\theta}; \alpha)\Big] \big|_{\theta = \bar{\theta}}. 
    \end{equation} 
\end{theorem}
In Section \ref{subsection_Experimentsonimprovementofoptimizationefficiencybyremovinggradientbias}, we empirically confirm that this unbiasedness is crucial for stable and effective optimization, as it prevents gradient estimates from drifting in the presence of heavy-tailed data.

\subsection{Addressing gradient vanishing problem}\label{subsection_Addressinggradientvanishingproblem} 

\begin{table*}[t] 
     \caption{Lower-boundedness status of $\alpha$-Div for each case of $\alpha < 0$, $\alpha > 1$, and $0 < \alpha < 1$.} 
     \label{Table_for_statusofthelowerboundedness_of_alphadiv_for_DRE} 
     \centering 
     \begin{tabular}{cccccc} 
         \toprule      
         Intervals of $\alpha$   & $\frac{1}{\alpha} \cdot \hat{E}_{Q[S]}$& $+$ & $\frac{1}{1- \alpha} \cdot \hat{E}_{P} \left[ e^{(\alpha - 1) \cdot T} \right]$ &$=$ &$\mathcal{L}_{\alpha\text{-Div}}(T; \alpha)$\\ 
         \midrule       
         $\alpha < 0$ & $\downarrow -\infty$ (as $e^T \uparrow \infty$) && $\ge 0$ & & lower-unbounded \\ 
         $\alpha > 1$ & $\ge 0$ && $\downarrow -\infty$ (as $e^T \uparrow \infty$) & & lower-unbounded \\ 
         $0 < \alpha < 1$ & $\ge 0$ & &$\ge 0$ & & \textbf{ lower-bounded } \\ 
         \bottomrule     
     \end{tabular} 
    \end{table*} 
  
    \begin{table*}[t] 
    \caption{Behavior of $E\big[\nabla_{\theta} \mathcal{L}_{\alpha\text{-standard}}(\phi_{\theta})\big]$ and $E\big[\nabla_{\theta} \mathcal{L}_{\alpha\text{-Div}}(T_{\theta})\big]$ as estimated probability ratios approach $0$ or $\infty$, for each case of $\alpha < 0$, $\alpha > 1$, and $0 < \alpha < 1$. The notations ``$\rightarrow \infty$'', ``$\rightarrow -\infty$'', and ``$\rightarrow \infty - \infty$'' indicate that at least one element of the gradient diverges positively, negatively, or becomes numerically unstable due to the subtraction of two diverging terms, respectively.}\label{Table_for_statusofgradientvanishing_alphadiv_for_DRE} 
    \centering 
    \begin{tabular}{ccccc} 
        \toprule            
        & 
        \multicolumn{2}{c}{$\big[\nabla_{\theta} \mathcal{L}_{\alpha\text{-standard}}(\phi_{\theta}) \big]\rightarrow  \text{?}$} & 
        \multicolumn{2}{c}{$E\big[\nabla_{\theta} \mathcal{L}_{\alpha\text{-Div}}(T_{\theta}) \big] 
            \rightarrow  \text{?} $} \\ 
         \cmidrule(lr){2-3} 
         \cmidrule(lr){4-5} 
        Intervals of $\alpha$  & 
        $E_P\big[\phi_{\theta}\big] \rightarrow 0$& 
        $E_P\big[\phi_{\theta} \big]\rightarrow \infty$& 
        $E_P\big[e^{T_{\theta}}\big] \rightarrow 0$& 
        $E_P\big[e^{T_{\theta}}\big] \rightarrow \infty$\\ 
        \midrule 
        $\alpha < 0$& 
        $\infty$& 
        $\mathbf{0}$& 
        $\infty - \infty$& 
        $\mathbf{0}$\\ 
        $\alpha > 1$ & 
        $\mathbf{0}$& 
        $-\infty$& 
        $\mathbf{0}$& 
        $\infty - \infty$\\ 
        $0 < \alpha < 1$& 
        $\infty$& 
        $\mathbf{0}$& 
        $- \infty$& 
        $\infty$\\ 
        \bottomrule          
    \end{tabular} 
\end{table*} 

When $\alpha$ is within $(0, 1)$, $\alpha$-Div avoids the gradient vanishing issue during training.  
Below, we describe why gradient vanishing does not occur in this case. 

First, we obtain the gradients of the standard $\alpha$-divergence loss in Equation (\ref{Eq_loss_func_standard_alpha_div}) and $\alpha$-Div: 
\begin{align} 
     \nabla_{\theta}	\mathcal{L}_{\alpha\text{-standard}} (\phi_{\theta}) &= 
     \hat{E}_{Q} \Big[ \nabla_{\theta} \phi_{\theta} \cdot {\phi_{\theta}}^{\alpha - 1} \Big] 
    - \hat{E}_{P} \Big[ \nabla_{\theta} \phi_\theta \cdot {\phi_{\theta}}^{\alpha - 2} \Big], \label{Eq_delta_standard_loss_alpha} \\ 
    \nabla_{\theta}	\mathcal{L}_{\alpha\text{-Div}} (T_{\theta}) &= 
     \hat{E}_{Q} \Big[ \nabla_{\theta} T_{\theta} \cdot e^{\alpha \cdot T_{\theta}} \Big] 
    - \hat{E}_{P} \Big[ \nabla_{\theta} T_{\theta} \cdot e^{(\alpha - 1) \cdot T_{\theta}} \Big]. \label{Eq_delta_loss_alpha} 
\end{align} 

Next, consider the case where the estimated probability ratios, $\phi_{\theta}$ and $e^{T_{\theta}}$, are either nearly zero or very large for some point $\mathbf{x}$. 
Because equations such that $E_{Q}[e^{T_{\theta}}] \rightarrow 0 \Leftrightarrow E_{P}[e^{T_{\theta}}] \rightarrow 0$ and $E_{Q}[e^{T_{\theta}}] \rightarrow \infty \Leftrightarrow E_{P}[e^{T_{\theta}}] \rightarrow \infty$ follow from  the assumption that $p(\mathbf{x}) > 0 \Leftrightarrow q(\mathbf{x}) > 0$ for all $\mathbf{x} \in \Omega$,
the behavior of $E\big[\nabla_{\theta}\mathcal{L}_{\alpha\text{-Div}}(T_{\theta}; \alpha)\big]$ under certain regularity conditions for $T_{\theta}$, 
as $E_{P}[e^{T_{\theta}}] \rightarrow 0$ or $E_{P}[e^{T_{\theta}}] \rightarrow \infty$, 
is summarized in Table \ref{Table_for_statusofgradientvanishing_alphadiv_for_DRE}.

In all cases except for $\alpha$-Div with $0 < \alpha < 1$, vanishing of the loss gradients 
is observed, such that $E\big[\nabla_{\theta} \mathcal{L}_{\alpha\text{-standard}}(\phi_{\theta}) \big] \rightarrow \mathbf{0}$ 
or $E\big[\nabla_{\theta} \mathcal{L}_{\alpha\text{-Div}}(T_{\theta}) \big] \rightarrow \mathbf{0}$. 
This implies that, during optimization, neural networks may remain stuck at extreme local minima 
when their density ratio estimations are either $0$ or $\infty$. 
However, this issue is avoided when $\alpha$ is within the interval $(0, 1)$. 
Additionally, choosing $\alpha$ within the interval $(0, 1)$ mitigates numerical instability arising from large differences in gradient values of the loss function for $\alpha > 1$ and $\alpha < 0$, which is represented as $E\big[\nabla_{\theta} \mathcal{L}_{\alpha\text{-Div}}(T{\theta}; \alpha)\big]\rightarrow \infty - \infty$ in Table \ref{Table_for_statusofgradientvanishing_alphadiv_for_DRE}.

\subsection{Sample size requirements for optimizing \texorpdfstring{$\alpha$}{α}-divergence}\label{subsection_Samplesizerequirementproblemforalphadivergence} 
We present the exact upper bound on the sample size required for minimizing $\alpha$-Div in Theorem \ref{theorem_consistency_alpha_div_est}, 
which corresponds to Equation (\ref{Eq_SamplesizerequirementproblemforKLdivergence}) for KL-divergence loss functions. 
The sample size requirement for minimizing $\alpha$-Div is upper-bounded depending on the value of $\alpha$. 
Intuitively, this property arises from the boundedness of Amari's $\alpha$-divergence: 
$0 \leq D_{\alpha} \leq 1/(\alpha \cdot(1-\alpha))$. 

\begin{theorem}\label{theorem_consistency_alpha_div_est} 
    Let $T^*= - \log (q(\mathbf{x})/p(\mathbf{x})) $ and $N=\min\{R, S\}$. Subsequently, let 
    \begin{equation} 
        \hat{D}^{(N)} (Q||P\, ;\, \alpha) = \frac{1}{\alpha\cdot(1-\alpha)} -  \mathcal{L}^{(N,N)}_{\alpha\text{-Div}}(T^*\, ;\, \alpha). 
    \end{equation} 
    Then, 
    \begin{equation} 
        \sqrt{N} \cdot \left\{ \hat{D}^{(N)} (Q||P\, ;\, \alpha) - D(Q||P\, ;\, \alpha) \right\} \xrightarrow{\ \  d \ \ } 
        \mathcal{N}\big(0, \sigma_{\alpha}\big) \label{Eq_theorem_consistency_alpha_div_est} 
    \end{equation} 
    holds, where 
    \begin{align} 
        \sigma_{\alpha}^2  &=  C^1_{\alpha} \cdot D(Q||P\,;\, 2\alpha) + C^2_{\alpha} \cdot  D(Q||P\,;\, 2 \alpha - 1) \nonumber\\ 
        & \ \, \quad  + C^3_{\alpha} \cdot  D(Q||P\, ;\, \alpha)^2 +  C^4_{\alpha} \cdot  D(Q||P\, ;\, \alpha) + C^5_{\alpha}, \nonumber\\ 
    \end{align} 
    and $C^1_{\alpha}=2\alpha \cdot (1 - 2\alpha)/\alpha^2$, $C^2_{\alpha} = 2\alpha \cdot (1 - 2\alpha)/(1-\alpha)^2$, 
    $C^3_{\alpha}= - 1/\alpha^2 - 1/(1 - \alpha)^2$, $C^4_{\alpha} = 2/\alpha^2 + 2/(1 - \alpha)^2$, 
    and  $C^5_{\alpha}= (1/\alpha^2+1/(1-\alpha)^2) \cdot (2 - 2 \alpha \cdot (1-\alpha))$. 
\end{theorem} 

Unfortunately, despite the sample requirement stated in Equation (\ref{Eq_theorem_consistency_alpha_div_est}), 
we empirically find that the estimation accuracy for $\alpha$-Div and KL-divergence loss functions is roughly the same as that of KL-divergence loss functions 
in downstream tasks of DRE, including KL-divergence estimation, as discussed in Section \ref{subsection_ExperimentsontheestimationaccuracyusinghighKLdivergencedata}.


\section{Experiments}\label{Section_NumericalExperiment} 

We evaluated the performance of our approach using synthetic datasets. 
First, we assessed the stability of the proposed loss function due to its lower-boundedness for $\alpha$ within $(0,1)$. 
Second, we validated the effectiveness of our approach in addressing the biased gradient issue in the training losses. 
Finally, we examined the $\alpha$-divergence loss function for DRE using high KL-divergence data. 
Details on the experimental settings and neural network training are provided in Section \ref{Section_Appendix_TheDetailsOfNumericalExperiments} in the Appendix. 

In addition to the results presented in this section, we conducted two additional experiments: 
a comparison of $\alpha$-Div with existing DRE methods, and experiments using real-world data. 
These additional experiments are reported in Section \ref{Section_Appendix_Additionalexperiments} in the Appendix. 

\subsection{Experiments on the Stability of Optimization for Different Values of \texorpdfstring{$\alpha$}{α}}\label{subsection_Experimentsonthestabilityinoptimizationfordifferentvaluesofalpha} 

We empirically confirmed the stability of optimization using $\alpha$-Div, as discussed in Section \ref{Section_AddressingTrainLossHackingProblem}. 
This includes addressing the potential divergence of training losses for $\alpha > 1$ and $\alpha < 0$, 
and observing the stability of optimization when $\alpha$ is within $(0, 1)$. 
Subsequently, we conducted experiments using synthetic datasets to examine the behavior of training losses 
during optimization across different values of $\alpha$ at each learning step. 

\textbf{Experimental Setup.} 
First, we generated 100 training datasets from two 5-dimensional normal distributions, 
$P=\mathcal{N}(\mu_p, I_5)$ and $Q=\mathcal{N}(\mu_q, \Sigma_q)$, 
where $\mu_p = \mu_q = (0, 0, \ldots, 0)$, and $I_5$ denotes the $5$-dimensional identity matrix. 
The covariance matrix $\Sigma_q = (\sigma_{ij})_{i=1}^5$ is defined as 
 $\sigma_{ii}=1$, and $\sigma_{ij}=0.8$ for $i\neq j$. 
Subsequently, we trained neural networks using the synthetic datasets by optimizing $\alpha$-Div for 
$\alpha = -3.0, \allowbreak -2.0, \allowbreak -1.0, \allowbreak 0.2, \allowbreak 0.5, \allowbreak 0.8, \allowbreak 2.0, \allowbreak 3.0$, and $4.0$, 
while measuring training losses at each learning step. 
For each value of $\alpha$, 100 trials were performed. 
Finally, we reported the median of the training losses at each learning step, along with the ranges between 
the 45th and 55th percentiles and between the 2.5th and 97.5th percentiles.

\textbf{Results.} 
Figure \ref{Figure_subsection_Experimentsonthestabilityinoptimizationfordifferentvaluesofalpha} presents the training losses of $\alpha$-Div across the learning steps for $\alpha = -2.0, \allowbreak 3.0$, and $0.5$. 
Results for other values of $\alpha$ are provided in Section \ref{Apdx_Section_thedetailsExperimentsonthestabilityinoptimizationfordifferentvaluesofalpha} in the Appendix. 
The figures on the left ($\alpha = -2.0$) and in the center ($\alpha = 3.0$) show that the training losses diverged to negative infinity 
when $\alpha < 0$ or $\alpha > 1$. 
In contrast, the figure on the right ($\alpha = 0.5$) demonstrates that the training losses successfully converged. 
These results highlight the stability of $\alpha$-Div's optimization when $\alpha$ is within the interval $(0, 1)$, 
as discussed in Section \ref{Section_AddressingTrainLossHackingProblem}. 

\subsection{Experiments on the Improvement of Optimization Efficiency by Removing Gradient Bias}
\label{subsection_Experimentsonimprovementofoptimizationefficiencybyremovinggradientbias} 

\begin{figure*}[t] 
  \begin{center} \centerline{\includegraphics[width=1.00\columnwidth]{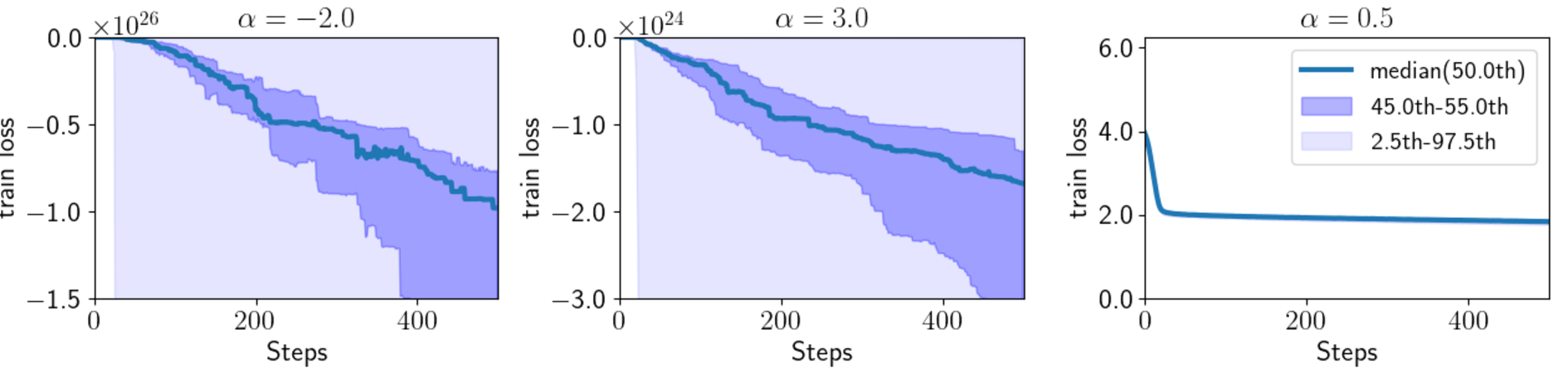}} 
  \caption{Results from Section \ref{subsection_Experimentsonthestabilityinoptimizationfordifferentvaluesofalpha}. The left ($\alpha = -2.0$), center ($\alpha = 3.0$), and right ($\alpha = 0.5$) graphs show training losses ($y$-axis) over learning steps ($x$-axis) during optimization using $\alpha$-Div with different $\alpha$ values. Solid blue lines represent median training losses, dark blue shaded areas show the 45th to 55th percentiles, and light blue shaded areas represent the 2.5th to 97.5th percentiles.} \label{Figure_subsection_Experimentsonthestabilityinoptimizationfordifferentvaluesofalpha} \end{center} 
   \begin{center}
     \centerline{\includegraphics[width=1.00\columnwidth]{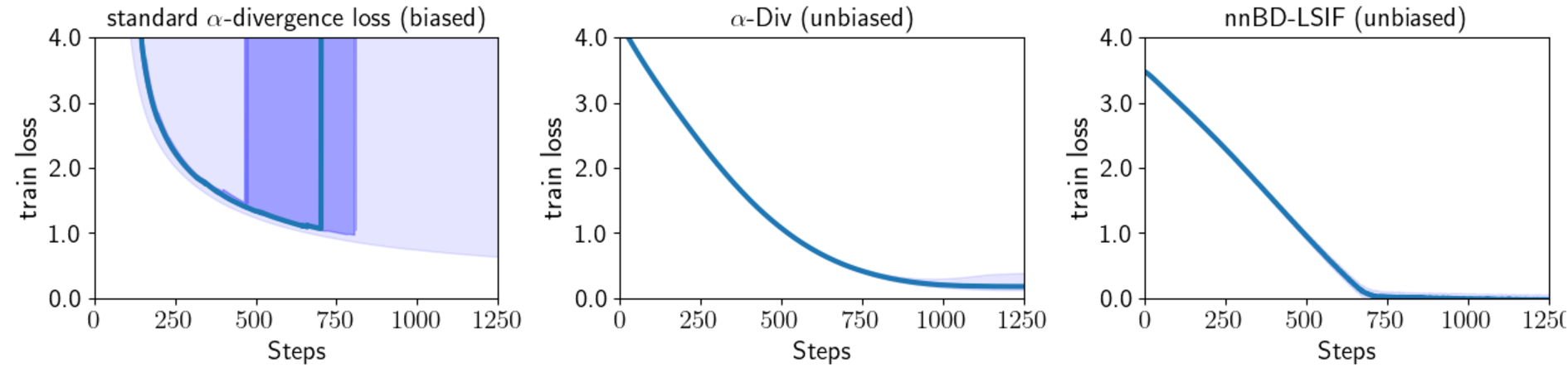}}
     \centerline{\includegraphics[width=1.00\columnwidth]{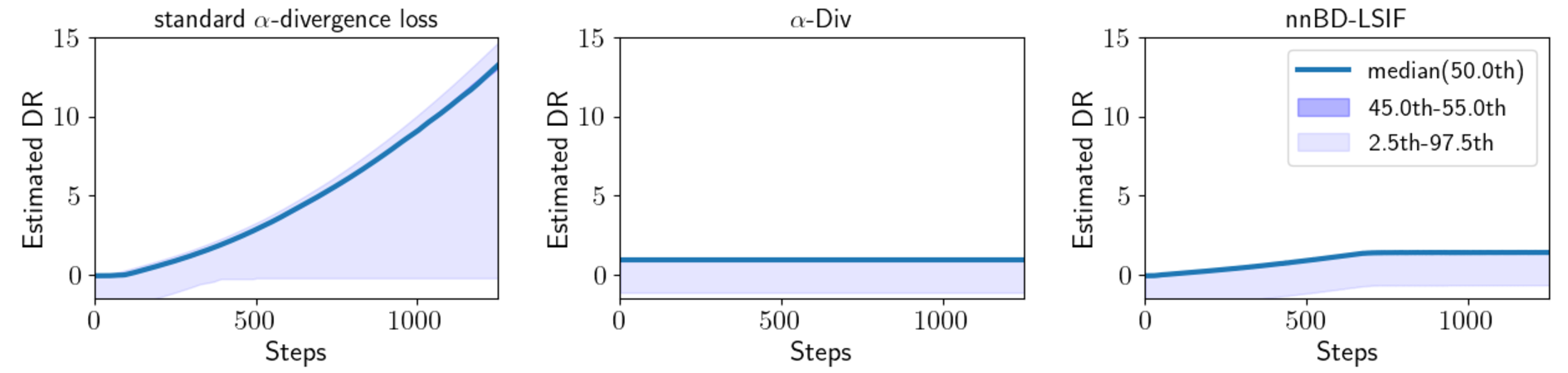}}   
     \caption{Results from Section \ref{subsection_Experimentsonimprovementofoptimizationefficiencybyremovinggradientbias}. The top row shows training losses, and the bottom row shows estimated density ratios (DR) during optimization.  The left column uses the standard $\alpha$-divergence loss function (biased gradients), the center column uses $\alpha$-Div (unbiased gradients), and the right column uses nnBD-LSIF (unbiased gradients). The $x$-axis represents learning steps. Solid blue lines indicate median values, dark blue shaded areas show the 45th to 55th percentiles, and light blue shaded areas represent the 2.5th to 97.5th percentiles.} \label{Figure_xperimentsonimprovementofoptimizationefficiencybyremovinggradientbias} 
   \end{center} 
    \vskip -0.3in 
\end{figure*} 

Unbiased gradients of loss functions are expected to optimize neural network parameters more effectively than biased gradients, 
since they update the parameters in ideal directions at each iteration.
We empirically compared the efficiency of minimizing training losses between the proposed loss function and the standard $\alpha$-divergence loss function derived from Equation (\ref{Eq_delta_standard_loss_alpha}), which highlighted the effectiveness of the unbiased gradients of the proposed loss function. 
Additionally, we observed that the estimated density ratios using the standard $\alpha$-divergence loss function diverged to large positive values, suggesting that the gradients of the standard $\alpha$-divergence loss function vanished. In contrast, $\alpha$-Div exhibited stable estimation.
This finding aligns with the discussion in Section \ref{subsection_Addressinggradientvanishingproblem}. 

\textbf{Experimental Setup.} We first generated 100 training datasets from two normal distributions, $P=\mathcal{N}(\mu_p, I_5)$ and $Q=\mathcal{N}(\mu_q, I_5)$, where $I_5$ denotes the 5-dimensional identity matrix. The means were set as $\mu_p = (-5/2, 0, 0, 0, 0)$ and $\mu_q = (5/2, 0, 0, 0, 0)$. We then trained neural networks using three different loss functions: the standard $\alpha$-divergence loss function defined in Equation (\ref{Eq_delta_standard_loss_alpha}), $\alpha$-Div, and deep direct DRE (D3RE) \citep{kato2021non}. Training losses were measured at each learning step. 
D3RE addresses train-loss hacking issues associated with Bregman divergence loss functions, as described in Section \ref{subsection_Vanishinggradientsproblem}, by mitigating the lower-unboundedness of loss functions. 
Specifically, for D3RE, we employed the neural network-based Bregman divergence Least Squares Importance Fitting (nnBD-LSIF) loss function, which ensures unbiased gradients and stable optimization. The hyperparameter for nnBD-LSIF was set to $C=2$. For both the standard $\alpha$-divergence loss and $\alpha$-Div, we used $\alpha = 0.5$. Finally, we reported the median training losses at each learning step, along with ranges between the 45th and 55th percentiles and between the 2.5th and 97.5th percentiles. 

\paragraph{Results.} 
The top row in Figure \ref{Figure_xperimentsonimprovementofoptimizationefficiencybyremovinggradientbias} illustrates the training losses at each learning step for each loss function. The center and right panels show that $\alpha$-Div and nnBD-LSIF are more effective at minimizing training losses compared to the standard $\alpha$-divergence loss function. These findings indicate that the unbiased gradient of $\alpha$-Div, like nnBD-LSIF, leads to more efficient neural network optimization than the biased gradient of the standard $\alpha$-divergence loss function. These results highlight the ineffectiveness in optimization of the standard $\alpha$-divergence loss function   inherent in its biased gradients and  $\alpha$-Div successfully mitigates this issue. 

Additionally, the training losses for the standard $\alpha$-divergence loss function diverged to positive infinity after 400 steps (the top panel in the left column), and the estimated density ratio diverged during optimization (the bottom panel in the left column). 
As shown in Table \ref{Table_for_statusofgradientvanishing_alphadiv_for_DRE}, 
the divergence of both the standard $\alpha$-divergence loss function and the estimated density ratio when $0 < \alpha < 1$, that is $\mathcal{L}_{\alpha\text{-standard}}(\phi_{\theta}) \rightarrow \infty$ and $\phi_{\theta} \rightarrow \infty$ with $0 < \alpha < 1$, imply that $E\big[\nabla_{\theta} \mathcal{L}_{\alpha\text{-standard}}(\phi_{\theta})\big] \rightarrow \mathbf{0}$. 
Consequently, these results demonstrate that the vanishing gradient issue in the standard $\alpha$-divergence loss function occurs when the estimated density ratio $\phi_{\theta}$ becomes  very large when $0 < \alpha < 1$. In contrast, $\alpha$-Div avoids this instability by maintaining stable gradients, which aligns with the discussion in Section \ref{subsection_Addressinggradientvanishingproblem}.

\subsection{Experiments on the Estimation Accuracy Using High KL-Divergence Data} \label{subsection_ExperimentsontheestimationaccuracyusinghighKLdivergencedata} 
\begin{figure*}[t] 
\begin{center} \centerline{\includegraphics[width=.70\columnwidth]{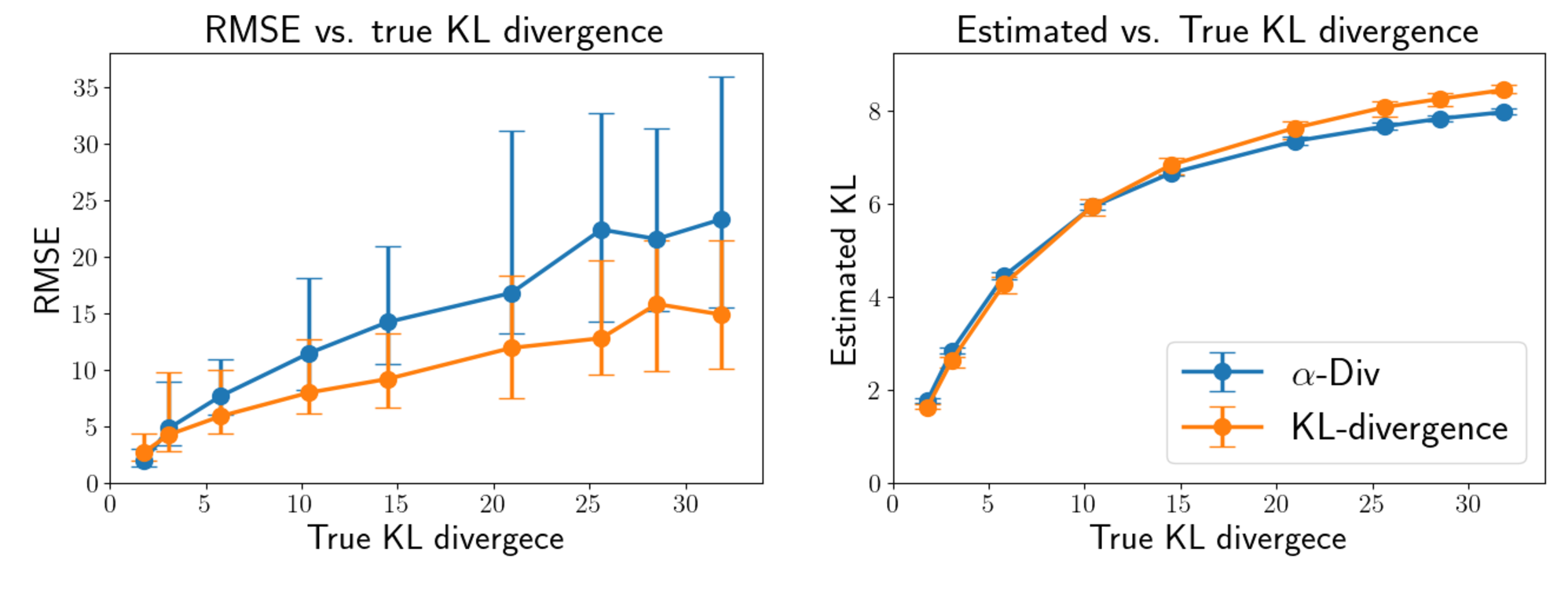}} \caption{Results of Section \ref{subsection_ExperimentsontheestimationaccuracyusinghighKLdivergencedata}. The $x$-axis represents the ground truth KL-divergence of the data. The $y$-axes of the left and right graphs represent the RMSE and estimated KL-divergence, respectively. The plot shows the median $y$-axis values for the ground truth KL-divergence. Vertical lines indicate the interquartile range (25th to 75th percentiles) of the $y$-axis values.} \label{Figure_ExperimentsontheestimationaccuracyusinghighKLdivergencedata} 
 \end{center} 
 \vskip -0.3in 
\end{figure*}

In Section \ref{subsection_SamplesizerequirementproblemforKLdivergence}, we examined the $\alpha$-divergence loss function, hypothesizing that its boundedness could address sample size issues in high KL-divergence data. Theorem \ref{theorem_consistency_alpha_div_est}, based on this boundedness, suggests that $\alpha$-Div can be minimized regardless of the true KL-divergence, indicating its potential for effective DRE with high KL-divergence data. To validate this hypothesis, we assessed DRE and KL-divergence estimation accuracy using both $\alpha$-Div and a KL-divergence loss function. 

However, we observed that the RMSE of DRE using $\alpha$-Div increased significantly with higher KL-divergence, similar to the KL-divergence loss function.

Additionally, both methods yielded nearly identical KL-divergence estimations. These findings suggest that the accuracy of DRE and KL-divergence estimation is primarily influenced by the true amount of KL-divergence in the data rather than the $\alpha$-divergence. 

\textbf{Experimental Setup.} We generated 100 training and 100 test datasets, each containing 10,000 samples. The datasets were drawn from two normal distributions, $P=\mathcal{N}(\mu_p, \sigma^2 \cdot I_3)$ and $Q=\mathcal{N}(\mu_q, 4^2 \cdot I_3)$, where $\mu_p=(-3/2, -3/2, -3/2)$ and $\mu_q=(3/2, 3/2, 3/2)$, with $I_3$ denoting the 3-dimensional identity matrix. The values of $\sigma$ were set to 1.0, 1.1, 1.2, 1.4, 1.6, 2.0, 2.5, and 3.0. Correspondingly, the ground truth KL-divergence values of the datasets were 31.8, 25.6, 21.0, 14.5, 10.4, 5.8, 3.1, and 1.8 nats\footnote{A 'nat' is a unit of information measured using the natural logarithm (base $e$)}, reflecting the increasing $\sigma^2$ values. The true density ratios of the test datasets are known for this experimental setup. We trained neural networks on the training datasets by optimizing both $\alpha$-Div with $\alpha=0.5$ and the KL-divergence loss function. After training, we measured the root mean squared error (RMSE) of the estimated density ratios using the test datasets. Additionally, we estimated the KL-divergence of the test datasets based on the estimated density ratios using a plug-in estimator. Finally, we reported the median RMSE of the DRE and the estimated KL-divergence, along with the interquartile range (25th to 75th percentiles), for both the KL-divergence loss function and $\alpha$-Div.

\textbf{Results.} Figure \ref{Figure_ExperimentsontheestimationaccuracyusinghighKLdivergencedata} shows the experimental results. 
The $x$-axis represents the true KL-divergence values of the test datasets, while the y-axes of the graphs display the RMSE (left) and estimated KL-divergence (right) for the test datasets. 
We empirically observed that the RMSE for DRE using $\alpha$-Div increased significantly as the KL-divergence of the datasets increased. A similar trend was observed for the KL-divergence loss function. Additionally, the KL-divergence estimation results were nearly identical between both methods. These findings indicate that the accuracy of DRE and KL-divergence estimation is primarily determined by the magnitude of KL-divergence in the data and is less influenced by $\alpha$-divergence. 
Therefore, we conclude that the approach discussed in Section \ref{subsection_Samplesizerequirementproblemforalphadivergence} offers no advantage over the KL-divergence loss function in terms of the RMSE for DRE with high KL-divergence data. 
However, we believe that these empirical findings contribute to a deeper understanding of the accuracy of downstream tasks in DRE using $f$-divergence loss functions.

\section{Conclusion}\label{conclusion} 
This study introduced a novel loss function for DRE, $\alpha$-Div, which is both concise and provides stable, efficient optimization. We offered technical justifications and demonstrated its effectiveness through numerical experiments. The empirical results confirmed the efficiency of the proposed loss function. However, experiments with high KL-divergence data revealed that the $\alpha$-divergence loss function did not offer a significant advantage over the KL-divergence loss function in terms of RMSE for DRE. These findings contribute to a deeper understanding of the accuracy of downstream tasks in DRE when using $f$-divergence loss functions.

\bibliography{references_paper_alphaDiv} 
\bibliographystyle{plainnat}

\newpage 

\appendix 

\section{Organization of the Supplementary Document} 
The organization of this supplementary document is as follows: 
Section \ref{SectionRelatedWork} reviews prior work in DRE using $f$-divergence optimization. 
Section \ref{Section_Appendix_proofs} presents the theorems and proofs cited in this study. 
Section \ref{Section_Appendix_TheDetailsOfNumericalExperiments} provides details of the numerical experiments conducted. 
Finally, Section \ref{Section_Appendix_Additionalexperiments} presents additional experimental results. 

\section{Related Work}\label{SectionRelatedWork} 
\citet{nguyen2010estimating} proposed DRE using variational representations of $f$-divergences. \citet{sugiyama2012density} introduced density-ratio matching under the Bregman divergence, a general framework that unifies various methods for DRE. 
As noted by \citet{sugiyama2012density}, density-ratio matching under the Bregman divergence is equivalent to DRE using variational representations of $f$-divergences. 
\citet{kato2021non} proposed a correction method for Bregman divergence loss functions s, in which the loss functions diverge to negative infinity, as discussed in Section \ref{subsection_TrainlossHackingproblem}.  
For estimation in scenarios with high KL-divergence data, \citet{rhodes2020telescoping} proposed a method that divides the high KL-divergence estimation into multiple smaller divergence estimations. \citet{choi2022density} further developed a continuous decomposition approach by introducing an auxiliary variable for transforming the data distribution. DRE using variational representations of $f$-divergences has also been studied from the perspective of classification-based modeling. \citet{menon2016linking} demonstrated that DRE via $f$-divergence optimization can be represented as a binary classification problem. \citet{kato2019learning} proposed using the risk functions in PU learning for DRE. 

Lastly, we review prior studies on DRE focusing on $\alpha$-divergence loss functions. 
\citet{birrell2021variational} derived an $\alpha$-divergence loss function from R\'enyi's $\alpha$-divergence, while \citet{cai2020utilizing} employed a standard variational representation of Amari's $\alpha$-divergence with $\alpha < 0$ or $\alpha > 1$. 
\citet{kwon2024alpha} presented essentially the same $\alpha$-divergence loss function proposed in this study, utilizing the Gibbs density expression to measure entropy in thermodynamics. In contrast, we propose the loss function to address the biased gradient problem.

\section{Proofs}\label{Section_Appendix_proofs} 
In this section, we present the theorems and the proofs referenced in this study. First, we define $\alpha$-Div within a probabilistic theoretical framework. Following that, we provide the theorems and the proofs cited throughout the study. 

\paragraph{Capital, small and bold letters.} 
Random variables are denoted by capital letters. For example, $X$. 
Small letters are used for values of the random variables corresponding to the capital letters; 
$a$ denotes a value of the random variable $X$. 
Bold letters $\mathbf{X}$ and $\mathbf{x}$ represent sets of random variables and their values. 

\subsection{Definition of \texorpdfstring{$\alpha$}{α}-Div}\label{Subsection_RigorousdefinitionofalphaDiv} 

\begin{definition}[$\alpha$-Divergence loss]\label{Definition_alphaDivergenceloss} 

    Let $\mathbf{X}^{1}_{P}, \mathbf{X}^{2}_{P}, \ldots, \mathbf{X}^{R}_{P}$ denote $R$ i.i.d. random variables drawn from $P$, and let 
    $\mathbf{X}^{1}_{Q}, \mathbf{X}^{2}_{Q}, \ldots, \mathbf{X}^{S}_{Q}$ denote $S$ i.i.d. random variables drawn from $Q$. 
    Then, the $\alpha$-Divergence loss $\mathcal{L}_{\alpha\text{-Div}}^{(R,S)}(\cdot\, ;\, \alpha)$ is defined as follows: 
    \begin{equation} 
        \mathcal{L}_{\alpha\text{-Div}}^{(R,S)}(T\, ;\, \alpha) 
        = \frac{1}{\alpha} \cdot  \frac{1}{S}  \cdot \sum_{i=1}^{S} e^{\alpha\cdot T(\mathbf{X}^{i}_{Q})}   + \frac{1}{1- \alpha} \cdot  \frac{1}{R} \cdot \sum_{i=1}^{R} e^{(\alpha -1 )\cdot T(\mathbf{X}^{i}_{P})},  \label{Appensix_Eq_loss_func_alpha_div} 
    \end{equation} 
    where $T$ is a measurable function over $\Omega$ such that $T:\Omega \rightarrow \mathbb{R}$. 
\end{definition} 

\subsection{Proofs for Section \ref{Section_DREusinganeuralnetworkwithanalphadivergenceloss}}\label{Section_Appendix_ProofsForSection_DREusinganeuralnetworkwithanalphadivergenceloss} 
In this section, we provide the theorems and the proofs referenced in Section \ref{Section_DREusinganeuralnetworkwithanalphadivergenceloss}. 

\begin{theorem}\label{Apdx_theorem_alpha_div_resp} 
    A variational representation of $\alpha$-divergence is given as 
    \begin{align} 
        D(Q||P\, ;\, \alpha) &= \sup_{\phi \ge 0} \left\{ 
        \frac{1}{\alpha\cdot(1-\alpha)} - \frac{1}{\alpha} \cdot E_{Q} \left[\phi^{-\alpha} \right]  - \frac{1}{1 - \alpha} \cdot E_{P} \left[\phi^{1 - \alpha} \right]  \right\}, \label{lemma_Eq_alpha_variation} 
    \end{align} 
    where the supremum is taken over all measurable functions with $ E_P[\phi^{1 - \alpha}] < \infty$ and $E_Q[\phi^{-\alpha}] < \infty$. 
    The maximum value is achieved at $\phi = dQ/dP$. 
\end{theorem} 
\begin{proof}[Proof of Theorem \ref{Apdx_theorem_alpha_div_resp}] 
    Let $f_{\alpha}(t) = \{t^{1-\alpha} - (1-\alpha)\cdot t - \alpha\}/\{\alpha \cdot (\alpha-1)\}$ for $\alpha \neq 0, 1$, then 
    \begin{align} 
        E_{P} \left[f_{\alpha} \left( \frac{dQ}{dP} \right) \right] 
        &= 	E_{P} \left[ \frac{1}{\alpha\cdot(1-\alpha)} \cdot \left( \frac{dQ}{dP} \right)^{1-\alpha} 
        + \frac{1}{\alpha} \cdot \left( \frac{dQ}{dP} \right) 
        + \frac{1}{1 - \alpha} \right] \nonumber\\ 
        &= \frac{1}{\alpha\cdot(1-\alpha)}\cdot  E_{P} \left[  \left( \frac{dQ}{dP} \right)^{1-\alpha} \right] 
        + \frac{1}{\alpha} +  \frac{1}{1 - \alpha}  \nonumber\\ 
        &= 	D(Q||P\, ;\, \alpha). 
    \end{align} 
    Note that, the Legendre transform of $g_{\alpha}(x) = x^{1-\alpha}/(1 - \alpha)$ is obtained as 
    \begin{equation} 
        g_{\alpha}^* (x)=  \frac{\alpha}{\alpha - 1} \cdot x^{1- \frac{1}{\alpha}}, \label{Eq_conjugate_alpha} 
    \end{equation} 
    and for the Legendre transforms of functions, it holds that 
    \begin{equation} 
        \{C \cdot h(x)\}^*= C \cdot h^*\left(\frac{x}{C}\right) \quad \text{and} \quad \{h(x) + C\cdot x + D\}^*= h^*(x - C) - D. \label{Eq_formula_conjugate} 
    \end{equation} 
    Here, $A^*$ denotes the Legendre transform of $A$. 

    From Equations (\ref{Eq_conjugate_alpha}) and (\ref{Eq_formula_conjugate}), we have 
    \begin{align} 
        f_{\alpha}^*(t) 
        &= \left\{ \frac{1}{(-\alpha)} \cdot g_{\alpha}(t) +  \frac{1}{\alpha} \cdot t + \frac{1}{1 - \alpha}  \right\}^* \nonumber \\ 
        &= \frac{1}{(-\alpha)}\cdot g_{\alpha}^*  \left( -\alpha \cdot \left\{ t- \frac{1}{\alpha}\right\} \right)		- \frac{1}{1 - \alpha} \nonumber \\ 
        &= - \frac{1}{\alpha} \cdot g_{\alpha}^*  \left(1  - \alpha t\right) + \frac{1}{\alpha - 1} \nonumber \\ 
        &= - \frac{1}{\alpha} \cdot \left\{ \frac{\alpha}{\alpha - 1} \cdot \left(1  - \alpha t\right)^{1- \frac{1}{\alpha}} \right\}   + \frac{1}{\alpha - 1} \nonumber \\ 
        &= \frac{1}{1 - \alpha} \cdot \left(1  - \alpha t\right)^{1- \frac{1}{\alpha}}  + \frac{1}{\alpha - 1}. \label{Eq_conjugate_f_alpha} 
    \end{align} 

    By differentiating $f_{\alpha}(t)$, we obtain 
    \begin{equation} 
        f_{\alpha}'(t) = -  \frac{1}{\alpha} \cdot t^{-\alpha} + \frac{1}{\alpha}. \label{Eq_derivative_f_alpha} 
    \end{equation} 
    Thus, 
    \begin{align} 
        E_Q \big[ f_{\alpha}'(\phi)  \big] =  E_Q\left[  - \frac{1}{\alpha}\cdot \phi^{-\alpha} + \frac{1}{\alpha} \right]. \label{Eq_Exp_f_derivative} 
    \end{align} 

    From (\ref{Eq_conjugate_f_alpha}) and (\ref{Eq_derivative_f_alpha}), we have 
    \begin{align} 
        E_P \big[ f_{\alpha}^*(f_{\alpha}'(\phi))  \big]  &= E_P\left[ 
        \frac{1}{1 - \alpha} \cdot \left\{ 1 -  \alpha \cdot \left( - \frac{1}{\alpha} \cdot \phi^{-\alpha} 
        + \frac{1}{\alpha} \right)  \right\}^{1- \frac{1}{\alpha}} + \frac{1}{\alpha - 1} \right] \nonumber \\ 
        &= E_P \left[ \frac{1}{1 - \alpha} \cdot \phi^{1 - \alpha} + \frac{1}{\alpha - 1} \right]. \label{Eq_Exp_f_star_f_derivative} 
    \end{align} 
    In addition, from Equations (\ref{Eq_Exp_f_derivative}) and (\ref{Eq_Exp_f_star_f_derivative}), 
    we observe that $ E_P\big[\phi^{1 - \alpha}\big] < \infty$ is equivalent to 
    $E_P \big[\, \big|f_{\alpha}^*(f_{\alpha}'(\phi)) \big| \,\big] < \infty$. 
    Similarly, $E_Q\big[\phi^{-\alpha}\big] < \infty$ is equivalent to $E_Q \big[\, \big|f_{\alpha}'(\phi) \big|\, \big] < \infty$. 

    Finally, by substituting Equations (\ref{Eq_Exp_f_derivative}) and (\ref{Eq_Exp_f_star_f_derivative}) into Equation (\ref{Eq_Variational_representation_f_div_phi}), we get 
    \begin{align} 
        D(Q||P\, ;\, \alpha) &= \sup_{\phi \ge 0}\Big\{ E_Q\big[f_{\alpha}'(\phi)\big] - E_P\big[f_{\alpha}^*(f_{\alpha}'(\phi)) \big]\Big\} \nonumber \\ 
        &= \sup_{\phi \ge 0} 
        \left\{ 
        E_Q \left[  - \frac{1}{\alpha} \cdot \phi^{-\alpha} + \frac{1}{\alpha} \right] 
        - E_P \left[ \frac{1}{1 - \alpha} \cdot \phi^{1 - \alpha} + \frac{1}{\alpha - 1} \right] 
        \right\} \nonumber \\ 
        &= \sup_{\phi \ge 0} \left\{ 
        \frac{1}{\alpha\cdot(1-\alpha)} 	- \frac{1}{\alpha}\cdot E_{Q} \left[\phi^{-\alpha} \right] 
        - \frac{1}{1 - \alpha}\cdot E_{P} \left[\phi^{1 - \alpha} \right] 
        \right\}.	 \nonumber 
    \end{align} 

    This completes the proof. 
\end{proof}

\begin{theorem}[Theorem \ref{theorem_alpha_div_resp_in_gibbs_dinsity_form} in Section \ref{Section_DREusinganeuralnetworkwithanalphadivergenceloss} restated] \label{Appendix_theorem_alpha_div_resp_in_gibbs_dinsity_form} 
    The $\alpha$-divergence is represented as 
    \begin{equation} 
        D(Q||P\, ;\, \alpha) = \sup_{T:\Omega \rightarrow \mathbb{R}} \left\{ 
        \frac{1}{\alpha\cdot(1-\alpha)} - \frac{1}{\alpha} \cdot E_{Q} \left[e^{\alpha \cdot T} \right] 
        - \frac{1}{1- \alpha}\cdot E_{P} \left[ e^{(\alpha - 1) \cdot T} \right] 	\right\}, \label{Lemma_Eq_loss_func_alpha_div} 
    \end{equation} 
    where the supremum is taken over all measurable functions $T:\Omega \rightarrow \mathbb{R}$ with 
    $E_P[e^{(\alpha - 1) \cdot T}] < \infty$ and $E_{Q} [e^{\alpha \cdot T}] < \infty$. 
    The equality holds for $T^*$ satisfying 
    \begin{equation} 
        \frac{dQ}{dP} = e^{-T^*}.  \label{restated_Eq_gibbs_densty_minus} 
    \end{equation} 
\end{theorem} 
\begin{proof}[proof of Theorem \ref{theorem_alpha_div_resp_in_gibbs_dinsity_form}] 
    Substituting $e^{- T}$ into $\phi$ in Equation (\ref{lemma_Eq_alpha_variation}), we have 
    \begin{align} 
        D(Q||P\, ;\, \alpha) &= \sup_{\phi \ge 0} \left\{ 
        \frac{1}{\alpha\cdot(1-\alpha)} - \frac{1}{\alpha} \cdot E_{Q} \left[\phi^{-\alpha} \right]  - \frac{1}{1 - \alpha} \cdot E_{P} \left[\phi^{1 - \alpha} \right]  \right\} \nonumber \\ 
        &= \sup_{T:\Omega \rightarrow \mathbb{R}} \left\{ 
        \frac{1}{\alpha\cdot(1-\alpha)}  - \frac{1}{\alpha}\cdot E_{Q} \left[\left\{ e^{-T} \right\}^{-\alpha} \right] 
        - \frac{1}{1 - \alpha} \cdot E_{P} \left[\left\{ e^{-T}\right\}^{1 - \alpha} \right] 
        \right\} \nonumber \\ 
        &= \sup_{T:\Omega \rightarrow \mathbb{R}} \left\{ 
        \frac{1}{\alpha\cdot(1-\alpha)} 	- \frac{1}{\alpha}\cdot E_{Q} \left[e^{\alpha \cdot T} \right] 
        - \frac{1}{1- \alpha}\cdot E_{P} \left[ e^{(\alpha - 1) \cdot T} \right] 
        \right\}. \label{Proof_lemma_opti_in_gibbs_Eq_state} 
    \end{align} 
    Finally, from Theorem \ref{theorem_alpha_div_resp}, the equality for Equation (\ref{Proof_lemma_opti_in_gibbs_Eq_state}) holds if and only if 
    \begin{equation} 
        \frac{dQ}{dP} = e^{-T^*}. 
    \end{equation} 

    This completes the proof. 
\end{proof}

\begin{lemma} \label{lemma_loss_optimal_T_is_log_density_ratio} 
    For a measurable function $T:\Omega \rightarrow \mathbb{R}$ with $E_P[e^{(\alpha - 1) \cdot T}] < \infty$ and $E_{Q} [e^{\alpha \cdot T}] < \infty$, let 
    \begin{align} 
        \widetilde{l}_{\alpha\text{-Div}} \left(T(\mathbf{x})\,;\, \alpha \right) &=  \frac{1}{\alpha} \cdot 
        e^{\alpha \cdot T(\mathbf{x})} \cdot \frac{dQ}{d\mu}(\mathbf{x}) +  \frac{1}{1- \alpha} 
        \cdot e^{(\alpha - 1) \cdot T(\mathbf{x})} \cdot \frac{dP}{d\mu}(\mathbf{x}). \label{Eq_lemma_loss_optimal_T_is_log_density_ratio} 
    \end{align} 
    Then the optimal function $T^*$ for $\inf_{T:\Omega \rightarrow \mathbb{R} } \widetilde{l}_{\alpha\text{-Div}} \left(T\,;\, \alpha \right)$ is obtained as $T^*= - \log  dQ/dP$, $\mu$-almost everywhere. 
\end{lemma} 
\begin{proof}[proof of Lemma \ref{lemma_loss_optimal_T_is_log_density_ratio}] 
    First, note that it follows from Jensen's inequality that 
    \begin{equation} 
        \log(p \cdot X + q \cdot Y) \ge p \cdot \log(X) + q \cdot \log(Y), \label{Eq_Jensensinequality} 
    \end{equation} 
    for $X, Y > 0$ and $p, q > 0$ with $p + q = 1$, and equality holds when $X = Y$. 

    Substitute $X = e^{\alpha \cdot T(\mathbf{x})} \cdot \frac{dQ}{d\mu}(\mathbf{x})$, 
    $Y = e^{(\alpha - 1) \cdot T(\mathbf{x})} \cdot \frac{dP}{d\mu}(\mathbf{x})$, 
    $p = 1 - \alpha$, and $q = \alpha$ into Equation (\ref{Eq_Jensensinequality}), we obtain 
    \begin{equation} 
        \log(p \cdot X + q \cdot Y) = \log\left(\frac{1}{\alpha \cdot (1-\alpha)} \cdot \widetilde{l}_{\alpha\text{-Div}}(T\,;\, \alpha)\right), \nonumber 
    \end{equation} 
    and $ \log\left(\frac{1}{\alpha \cdot (1-\alpha)} \cdot \widetilde{l}_{\alpha\text{-Div}}(T\,;\, \alpha)\right)$ is minimized when 
    $e^{\alpha \cdot T(\mathbf{x})} \cdot \frac{dQ}{d\mu}(\mathbf{x}) = e^{(\alpha - 1) \cdot T(\mathbf{x})} \cdot \frac{dP}{d\mu}(\mathbf{x})$, 
    $\mu$-almost everywhere. 
    Therefore, $\inf_{T:\Omega \rightarrow \mathbb{R}} \widetilde{l}_{\alpha\text{-Div}}(T\,;\, \alpha)$ is achieved at $e^{- T^*(\mathbf{x})} = \frac{dQ}{dP}(\mathbf{x})$, $\mu$-almost everywhere. 
    Thus, we obtain $T^*(\mathbf{x}) = - \log\frac{dQ}{dP}(\mathbf{x})$, $\mu$-almost everywhere. 

    This completes the proof. 
\end{proof}

\begin{lemma} \label{lemma_loss_not_biased_lemma} 
    For a measurable function $T:\Omega \rightarrow \mathbb{R}$ with $E_P[e^{(\alpha - 1) \cdot T}] < \infty$ and $E_{Q} [e^{\alpha \cdot T}] < \infty$, let 
    \begin{align} 
        \widebar{\mathcal{L}}_{\alpha\text{-Div}}(T\,;\,\alpha) &= E_{\mu}\left[ \widetilde{l}_{\alpha\text{-Div}} \left(T(\mathbf{x}\,;\, \alpha) \right) \right] \label{Eq_lemma_loss_not_biased_lemma_loss_T_z} \\ 
        &=  \frac{1}{\alpha} \cdot E_{Q} \left[ e^{\alpha \cdot T(\mathbf{x})} \right] +  \frac{1}{1 - \alpha} \cdot E_{P} \left[ e^{(\alpha - 1) \cdot T(\mathbf{x})} \right],   \label{Eq_lemma_loss_not_biased_lemma_loss_T} 
    \end{align} 
    and let 
    \begin{align} 
        \widetilde{l}_{\alpha\text{-Div}}^*(\alpha) &= \inf_{T:\Omega \rightarrow \mathbb{R} }\widetilde{l}_{\alpha\text{-Div}}(T\,;\, \alpha), \quad \text{and} \label{eq_lemma_loss_not_biased_lemma_inf_l}\\ 
        \widebar{\mathcal{L}}_{\alpha\text{-Div}}^*(\alpha) &=  \inf_{T:\Omega \rightarrow \mathbb{R} }\widebar{\mathcal{L}}_{\alpha\text{-Div}}(T\,;\,\alpha), \label{eq_lemma_loss_not_biased_lemma_inf_L} 
    \end{align} 
    where the infima of Equations (\ref{eq_lemma_loss_not_biased_lemma_inf_l}) and (\ref{eq_lemma_loss_not_biased_lemma_inf_L}) are considered over measurable functions $T:\Omega \rightarrow \mathbb{R}$ with $E_P[e^{(\alpha - 1) \cdot T}] < \infty$ and $E_Q[e^{\alpha \cdot T}] < \infty$. 

    Then, 
    \begin{equation} 
        E_{\mu} \left[ \widetilde{l}_{\alpha\text{-Div}}^*(\alpha) \right] = \widebar{\mathcal{L}}_{\alpha\text{-Div}}^*(\alpha). \label{Lemma_Eq_state_inf_t_l_t} 
    \end{equation} 
    Additionally, the equality in Equations (\ref{eq_lemma_loss_not_biased_lemma_inf_l}) and (\ref{eq_lemma_loss_not_biased_lemma_inf_L}) hold for $T^*(\mathbf{x}) = - \log  dQ/dP(\mathbf{x})$. 
\end{lemma} 
\begin{proof}[proof of Lemma \ref{lemma_loss_not_biased_lemma}] 
    Let  $T^*(\mathbf{x}) = - \log  dQ/dP(\mathbf{x})$. 

    First, it follows from Lemma \ref{lemma_loss_optimal_T_is_log_density_ratio} that 
    \begin{equation} 
        \widetilde{l}_{\alpha\text{-Div}}^*(\alpha) = \inf_{T:\Omega \rightarrow \mathbb{R}} \widetilde{l}_{\alpha\text{-Div}}(T\,;\, \alpha) = \widetilde{l}_{\alpha\text{-Div}}(T^*\,;\, \alpha). 
        \label{proof_lemmma_loss_not_biased_lemma_to_get_upper_bound_x1} 
    \end{equation} 

    Next, we obtain 
    \begin{align} 
        E_{\mu} \left[  \frac{1}{\alpha\cdot(1-\alpha)} -  \widetilde{l}_{\alpha\text{-Div}}^*(\alpha) \right] 
        &= \int \left\{ \frac{1}{\alpha\cdot(1-\alpha)} -  \widetilde{l}_{\alpha\text{-Div}}^* (\alpha)\right\} d\mu \nonumber\\ 
        &=  \int 
        \left\{ \frac{1}{\alpha\cdot(1-\alpha)} - \inf_{T:\Omega \rightarrow \mathbb{R} } \widetilde{l}_{\alpha\text{-Div}}(T\,;\, \alpha)\right\} d\mu \nonumber\\ 
        &= \int \sup_{T:\Omega \rightarrow \mathbb{R} } 
        \left\{ \frac{1}{\alpha\cdot(1-\alpha)} - \widetilde{l}_{\alpha\text{-Div}}(T\,;\, \alpha) \right\} d\mu. 
        \label{proof_lemmma_loss_not_biased_lemma_to_get_upper_bound} 
    \end{align} 


    Let $T_k = T^* + 1/k$. 
    Note that, 
    \begin{equation} 
        \lim_{k \rightarrow \infty}\widetilde{l}_{\alpha\text{-Div}}(T_k\,;\, \alpha) 
        = \inf_{T:\Omega \rightarrow \mathbb{R}} \widetilde{l}_{\alpha\text{-Div}}(T\,;\, \alpha) = \widetilde{l}_{\alpha\text{-Div}}^*(\alpha). 
    \end{equation} 
    Then,  we have 
    \begin{align} 
        &\lim_{k \rightarrow \infty}  \bigg\{ \frac{1}{\alpha\cdot(1-\alpha)} - \widetilde{l}_{\alpha\text{-Div}}(T_k\,;\, \alpha) \bigg\} \nonumber\\ 
        &= \frac{1}{\alpha\cdot(1-\alpha)} - \lim_{k \rightarrow \infty}   \bigg\{ \widetilde{l}_{\alpha\text{-Div}}(T_k\,;\, \alpha)\bigg\} \nonumber\\ 
        &= \frac{1}{\alpha\cdot(1-\alpha)} - \inf_{T:\Omega \rightarrow \mathbb{R}} \widetilde{l}_{\alpha\text{-Div}}(T\,;\, \alpha) \nonumber\\ 
        &=  \sup_{T:\Omega \rightarrow \mathbb{R}} \bigg\{ \frac{1}{\alpha\cdot(1-\alpha)}  -  \widetilde{l}_{\alpha\text{-Div}}(T\,;\, \alpha) \bigg\}.  \label{Eq_Lemma_loss_not_biased_lemma_N_left_1} 
    \end{align} 

    From Theorem \ref{Appendix_theorem_alpha_div_resp_in_gibbs_dinsity_form}, we have 
    \begin{align} 
       & \lim_{k \rightarrow \infty}  E_{\mu} \left[ \frac{1}{\alpha\cdot(1-\alpha)} - \widetilde{l}_{\alpha\text{-Div}}(T_k\,;\, \alpha) \right] \nonumber\\ 
        &= \frac{1}{\alpha\cdot(1-\alpha)} - \lim_{k \rightarrow \infty}    E_{\mu} \left[ \widetilde{l}_{\alpha\text{-Div}}(T_k\,;\, \alpha) \right] \nonumber\\ 
        &= \frac{1}{\alpha\cdot(1-\alpha)} - \lim_{k \rightarrow \infty} \left\{ \frac{1}{\alpha} \cdot E_{Q} \left[e^{\alpha \cdot T_k} \right]  +   \frac{1}{1- \alpha} \cdot E_{P} \left[ e^{(\alpha - 1) \cdot T_k} \right]  \right\} \nonumber\\ 
        &= \frac{1}{\alpha\cdot(1-\alpha)} - \inf_{T:\Omega \rightarrow \mathbb{R}} \left\{ \frac{1}{\alpha} \cdot E_{Q} \left[e^{\alpha \cdot T} \right]  +   \frac{1}{1- \alpha} \cdot E_{P} \left[ e^{(\alpha - 1) \cdot T} \right]  \right\} \nonumber\\ 
        &= \frac{1}{\alpha\cdot(1-\alpha)} - \inf_{T:\Omega \rightarrow \mathbb{R}} 
        E_{\mu} \left[ \widetilde{l}_{\alpha\text{-Div}}(T\,;\, \alpha)\right] \nonumber\\ 
        &= \sup_{T:\Omega \rightarrow \mathbb{R}}  \left\{\frac{1}{\alpha\cdot(1-\alpha)} - E_{\mu} \left[ \widetilde{l}_{\alpha\text{-Div}}(T\,;\, \alpha)\right] \right\} \nonumber\\ 
        &= \sup_{T:\Omega \rightarrow \mathbb{R}}  E_{\mu} \left[ \frac{1}{\alpha\cdot(1-\alpha)} - \widetilde{l}_{\alpha\text{-Div}}(T\,;\, \alpha) \right].  \label{Eq_Lemma_loss_not_biased_lemma_N_left_2} 
    \end{align} 

    Now, we have 
    \begin{align} 
        &\left| \frac{1}{\alpha\cdot(1-\alpha)} - \widetilde{l}_{\alpha\text{-Div}}(T_k\,;\, \alpha) \right| \nonumber\\ 
        &= \ \left| \frac{1}{\alpha\cdot(1-\alpha)} - \frac{1}{\alpha} \cdot \left( \frac{dQ}{dP} \left(\mathbf{x}\right)\right)^{\alpha} \cdot \frac{dQ}{d\mu} \left(\mathbf{x}\right) \cdot e^{\frac{\alpha}{k}} \right. \nonumber\\ 
        &\  \qquad \left. - \frac{1}{1- \alpha}\cdot  \left(\frac{dQ}{dP} \left(\mathbf{x}\right)\right)^{\alpha - 1} \cdot \frac{dP}{d\mu} \left(\mathbf{x}\right) \cdot e^{\frac{\alpha - 1}{k}} \right| \nonumber \\ 
        &\leq \ \frac{1}{\alpha\cdot(1-\alpha)} 
        + \frac{1}{\alpha} \cdot e^{\frac{\alpha}{k}} \cdot \left(\frac{dQ}{dP} \left(\mathbf{x}\right) \right)^{\alpha} \cdot \frac{dQ}{d\mu} \left(\mathbf{x}\right) \nonumber\\ 
        & \ \qquad + \frac{1}{1-\alpha} \cdot e^{\frac{\alpha - 1}{k}} \cdot \left( \frac{dQ}{dP} \left(\mathbf{x}\right)\right)^{\alpha - 1} \cdot \frac{dP}{d\mu} \left(\mathbf{x}\right). \nonumber \\ 
        &=  \frac{1}{\alpha\cdot(1-\alpha)} +  \left\{ \frac{1}{\alpha} \cdot e^{\frac{\alpha}{k}}  +  \frac{1}{1-\alpha} \cdot e^{\frac{\alpha - 1}{k}}    \right\} \cdot \left( \frac{dQ}{dP} \left(\mathbf{x}\right)\right)^{\alpha - 1} \cdot \frac{dP}{d\mu}\left(\mathbf{x}\right), 
          \label{Lemma_loss_not_biased_lemma_integral_is_finte} 
    \end{align} 
    and let $\phi(\mathbf{x})$ denote the term on the right hand side of Equation (\ref{Lemma_loss_not_biased_lemma_integral_is_finte}). 

    Then, we observe that 
    \begin{equation} 
        \left|\frac{1}{\alpha\cdot(1-\alpha)} - \widetilde{l}_{\alpha\text{-Div}}(T_k(\mathbf{x})\,;\, \alpha) \right| \leq \phi(\mathbf{x}) \quad \text{and}  \quad E_{\mu} \big[\phi(\mathbf{x})\big] < \infty. 
        \nonumber 
    \end{equation} 
    That is, the following sequence is uniformly integrable for $\mu$: 
    \begin{equation} 
        \left\{ \frac{1}{\alpha\cdot(1-\alpha)} - \widetilde{l}_{\alpha\text{-Div}}( T_k\,;\, \alpha) \right\}_{k=1}^{\infty}. 
        \nonumber 
    \end{equation} 
    Thus, from the property of the Lebesgue integral (\citeauthor{shiryaev1995probability}, P188, Theorem 4), we have 
    \begin{equation} 
        E_{\mu} \left[ \lim_{k \rightarrow \infty} \left\{ \frac{1}{\alpha\cdot(1-\alpha)} - \widetilde{l}_{\alpha\text{-Div}}( T_k\,;\, \alpha) \right\} \right] 
        = \lim_{k \rightarrow \infty} E_{\mu} \left[ \frac{1}{\alpha\cdot(1-\alpha)} - \widetilde{l}_{\alpha\text{-Div}}( T_k\,;\, \alpha) \right]. \label{exchange_integral_sup} 
    \end{equation}

    From Equations (\ref{Eq_Lemma_loss_not_biased_lemma_N_left_1}), (\ref{Eq_Lemma_loss_not_biased_lemma_N_left_2}) and (\ref{exchange_integral_sup}), we 
    obtain 
    \begin{align} 
        \frac{1}{\alpha\cdot(1-\alpha)} - E_{\mu} \left[  \widetilde{l}_{\alpha\text{-Div}}^*(\alpha) \right] 
        &=  E_{\mu} \left[  \frac{1}{\alpha\cdot(1-\alpha)} -  \widetilde{l}_{\alpha\text{-Div}}^*(\alpha)  \right] \nonumber  \\ 
        &=  E_{\mu} \left[ \sup_{T:\Omega \rightarrow \mathbb{R} } 
        \bigg\{ \frac{1}{\alpha\cdot(1-\alpha)}  -   \widetilde{l}_{\alpha\text{-Div}}(T\,;\, \alpha) 
        \bigg\}  \ \right] \nonumber  \\ 
        &= E_{\mu} \left[   \lim_{k \rightarrow \infty} \left\{ \frac{1}{\alpha\cdot(1-\alpha)} -  \widetilde{l}_{\alpha\text{-Div}}(T_k\,;\, \alpha) \right\} \right] \nonumber  \\ 
        & \qquad \qquad \qquad \qquad \qquad \qquad \qquad \qquad \qquad \quad \text{($\therefore$ Equation (\ref{Eq_Lemma_loss_not_biased_lemma_N_left_1}))} \nonumber \\ 
        &= \lim_{k \rightarrow \infty} E_{\mu} \left[  \frac{1}{\alpha\cdot(1-\alpha)} -  \widetilde{l}_{\alpha\text{-Div}}(T_k\,;\, \alpha) \right] 
        \qquad \text{($\therefore$  Equation (\ref{exchange_integral_sup}))} \nonumber\\ 
        \nonumber\\ 
        &= \sup_{T:\Omega \rightarrow \mathbb{R} }\left\{ E_{\mu} \left[ \frac{1}{\alpha\cdot(1-\alpha)}  -  \widetilde{l}_{\alpha\text{-Div}}(T_k\,;\, \alpha) \right]  \right\} \nonumber\\ 
        & \qquad \qquad \qquad \qquad \qquad \qquad \qquad \qquad \text{($\therefore$  Equation (\ref{Eq_Lemma_loss_not_biased_lemma_N_left_2}))} \nonumber \\ 
        &= \frac{1}{\alpha\cdot(1-\alpha)}  -   \inf_{T:\Omega \rightarrow \mathbb{R} } E_{\mu} \left[ \widetilde{l}_{\alpha\text{-Div}}(T_k\,;\, \alpha) \right] \nonumber \\ 
        &= \frac{1}{\alpha\cdot(1-\alpha)}  -  \widebar{\mathcal{L}}_{\alpha\text{-Div}}^*(\alpha). \nonumber 
    \end{align} 

    Here, we have 
    \begin{equation} 
        E_{\mu} \left[ \widetilde{l}_{\alpha\text{-Div}}^*(\alpha) \right] = \widebar{\mathcal{L}}_{\alpha\text{-Div}}^*(\alpha).  \label{Lemma_loss_not_biased_lemma_N_for_final_state_left_1} 
    \end{equation} 

    From Equations 
    (\ref{Eq_lemma_loss_not_biased_lemma_loss_T_z}) and 
     (\ref{Lemma_loss_not_biased_lemma_N_for_final_state_left_1}, we have 
    \begin{equation} 
          \widebar{\mathcal{L}}_{\alpha\text{-Div}}(T^*\,;\, \alpha) =  E_{\mu} \left[ \widetilde{l}_{\alpha\text{-Div}}^*(\alpha) \right] = \widebar{\mathcal{L}}_{\alpha\text{-Div}}^*(\alpha). 
        \label{proof_lemmma_loss_not_biased_lemma_to_get_upper_bound_x2} 
    \end{equation}

    This completes the proof. 
\end{proof} 

\begin{theorem}[Theorem \ref{theorem_alpha_div_loss_in_const_shift} in Section \ref{Section_DREusinganeuralnetworkwithanalphadivergenceloss} restated] \label{Appendix_theorem_alpha_div_loss_in_const_shift} 
    For a fixed function $T:\Omega \rightarrow \mathbb{R}$, 
    let  $c_*$ be the optimal scalar value for the following infimum: 
    \begin{align} 
        c_* &=  \arg \inf_{c \in \mathbb{R}} E[\mathcal{L}_{\alpha\text{-Div}}(T + c\, ;\, \alpha)] \nonumber \\ 
        &=  \arg \inf_{c \in \mathbb{R}} \left\{ 
        \frac{1}{\alpha} E_{Q} \left[e^{\alpha \cdot(T + c)} \right]  \right. \nonumber \\ 
        & \left. 
        \quad \qquad   +  \quad \frac{1}{1- \alpha} E_{P} \left[ e^{(\alpha - 1) \cdot (T + c)} \right] \right\}, 
        \label{Eq_apdx_theorem_alpha_div_loss_for_training_in_const_shift} 
    \end{align} 
    Then,  $c_*$ satisfies $e^{c_*} = E_{P} \left[ e^{- T} \right]$, or 
    equivalently, $e^{- (T + c_*)} = e^{- T} / E_{P} \left[ e^{- T} \right]$. 
\end{theorem} 
\begin{proof}[proof of Theorem \ref{Appendix_theorem_alpha_div_loss_in_const_shift}] 
    Now, we have 
    \begin{equation} 
        \widetilde{l}_{\alpha\text{-Div}}(T + c\,;\, \alpha) = 
        \frac{1}{\alpha}  \cdot e^{\alpha \cdot c} \cdot e^{\alpha \cdot T(\mathbf{x})} \cdot \frac{dQ}{d\mu}(\mathbf{x}) 
        + \frac{1}{1 - \alpha}  \cdot e^{(\alpha - 1) \cdot c} \cdot e^{(\alpha - 1)  \cdot T(\mathbf{x})} \cdot \frac{dP}{d\mu}(\mathbf{x}). \nonumber 
    \end{equation} 

    For Equation (\ref{Eq_Jensensinequality}), let 
    $X = e^{\alpha \cdot c} \cdot e^{\alpha \cdot T(\mathbf{x})} \cdot  \frac{dQ}{d\mu}(\mathbf{x})$, 
    $Y = e^{(\alpha - 1) \cdot c} \cdot e^{(\alpha - 1)\cdot T(\mathbf{x})} \cdot \frac{dP}{d\mu}(\mathbf{x})$, 
    $p=1 - \alpha$ and $q=\alpha$. 
    Then, from Jensen's inequality, $\widetilde{l}_{\alpha\text{-Div}} \left(T + c\,;\, \alpha\right)$ is minimized at $c_*$ 
    such taht $e^{\alpha \cdot c_*} \cdot e^{\alpha \cdot T(\mathbf{x})} \cdot  \frac{dQ}{d\mu}(\mathbf{x}) = e^{(\alpha - 1) \cdot c_*} \cdot e^{(\alpha - 1)\cdot T(\mathbf{x})} \cdot \frac{dP}{d\mu}(\mathbf{x})$, 
    $\mu$-almost everywhere. 

    Hence, 
    \begin{equation} 
        e^{c_*} \cdot \frac{dQ}{d\mu}(\mathbf{x}) =  e^{- T(\mathbf{x})} \cdot \frac{dP}{d\mu}(\mathbf{x}). \nonumber 
    \end{equation} 
    By integrating both sides of the above equality over $\Omega$ with $\mu$, we obtain 
    \begin{equation} 
        e^{c_*} =  E_P\left[ e^{- T}  \right]. \nonumber 
    \end{equation} 

    This completes the proof. 
\end{proof}

\subsection{Proofs for Section \ref{Section_TheoreticaljustificationsofalphaDiv}}\label{Appendix_Section_proof_TheoreticaljustificationsofalphaDiv} 
In this section, we provide the theorems and the proofs referenced to in Section \ref{Section_TheoreticaljustificationsofalphaDiv}. 

\begin{theorem}[Theorem \ref{theorem_gradient_unbiased} in Section \ref{Section_TheoreticaljustificationsofalphaDiv} restated] \label{Appendixtheorem_gradient_unbiased} 
    Let $T_{\theta}(\mathbf{x}):\Omega \rightarrow \mathbb{R}$ be a function such that 
    the map $\theta = (\theta_1,\theta_2, \ldots, \theta_p) \in \Theta \mapsto T_{\theta}(\mathbf{x})$ is differentiable for all $\theta$ and for $\mu$-almost every $\mathbf{x} \in \Omega$. 
    Assume for a point $\bar{\theta} \in \Theta$, it holds that 
    $E_{P}[e^{(\alpha -1) \cdot T_{\bar{\theta}}}] < \infty$ and $E_{Q}[e^{\alpha \cdot T_{\bar{\theta}}}] < \infty$, 
    and there exists a compact neighborhood of $\bar{\theta}$, denoted by $B_{\bar{\theta}}$, 
    and a constant value $L$, such that $|T_{\psi}(\mathbf{x}) - T_{\bar{\theta}}(\mathbf{x})| < L \|\psi - \bar{\theta}\|$. 

    Then, 
    \begin{equation} 
        E \left[ \nabla_{\theta} \,\mathcal{L}_{\alpha\text{-Div}}^{(R,S)}(T\,;\, \alpha) \Big|_{\theta = \bar{\theta}} \right] 
        =  \nabla_{\theta} \, E \left[ \mathcal{L}_{\alpha\text{-Div}}^{(R,S)}(T\,;\, \alpha)   \right] \Big|_{\theta = \bar{\theta}}. \label{Prop_Eq_state} 
    \end{equation} 
\end{theorem} 
\begin{proof}[proof of Theorem \ref{Appendixtheorem_gradient_unbiased}] 
    We now consider the values, as $\psi \rightarrow \bar{\theta}$, of the following two integrals: 
    \begin{equation} 
        \int  \frac{1}{\|\psi - \bar{\theta}\|} \left\{ \frac{1}{\alpha}  e^{\alpha \cdot T_{\psi}} - \frac{1}{\alpha} e^{\alpha \cdot T_{\bar{\theta}}}  \right\}  dQ, 
        \label{proof_propos_grad_non_bias_q} 
    \end{equation} 
    and 
    \begin{equation} 
        \int  \frac{1}{\|\psi - \bar{\theta}\|} \left\{ \frac{1}{1-\alpha}  e^{(\alpha-1) \cdot T_{\psi}} - \frac{1}{1-\alpha} e^{(\alpha-1) \cdot T_{\bar{\theta}}}  \right\} dP. 
        \label{proof_propos_grad_non_bias_p} 
    \end{equation} 

    Note that it follows from the intermediate value theorem that 
    \begin{equation} 
        \left|  \frac{1}{\alpha}  e^{\alpha \cdot x} - \frac{1}{\alpha} e^{\alpha \cdot y} \right| 
        = \big|x-y\big| \cdot  e^{\alpha \cdot \{ y + \tau \cdot (x-y)\} }  \quad (\, \exists \tau \in [0, 1] \,). \label{proof_propos_grad_non_bias_eq_intermediate_theorem} 
    \end{equation} 

    By using the above equation with $x=T_{\psi}(\mathbf{x})$ and $y=T_{\bar{\theta}}(\mathbf{x})$ for the integrand of Equation (\ref{proof_propos_grad_non_bias_q}), we have 
    \begin{align} 
        &\left|  \frac{1}{\|\psi - \bar{\theta}\|} \cdot \left\{ \frac{1}{\alpha}  e^{\alpha \cdot T_{\psi}(\mathbf{x})} 
        - \frac{1}{\alpha} e^{\alpha \cdot T_{\bar{\theta}}(\mathbf{x})}  \right\} \right| \nonumber\\ 
        &=   \frac{1}{\|\psi - \bar{\theta}\|}  \cdot \big|T_{\psi}(\mathbf{x}) - T_{\bar{\theta}} (\mathbf{x})\big| 
        \cdot  e^{\alpha \cdot \{ T_{\bar{\theta}}(\mathbf{x}) +  \tau_{\mathbf{x}}  \cdot (T_{\psi}(\mathbf{x}) - T_{\bar{\theta}}(\mathbf{x}))  \}} 
        \quad \qquad (\, \tau_{\mathbf{x}} \in [0, 1] \,)  \nonumber\\ 
        &=   \frac{1}{\|\psi - \bar{\theta}\|} \cdot  \big|T_{\psi}(\mathbf{x}) - T_{\bar{\theta}}(\mathbf{x}) \big| 
        \cdot e^{\alpha  \cdot  \tau_{\mathbf{x}}  \cdot(T_{\psi}(\mathbf{x}) - T_{\bar{\theta}}(\mathbf{x}))} 
        \cdot  e^{\alpha \cdot T_{\bar{\theta}}(\mathbf{x})}   \nonumber\\ 
        &\leq   \frac{1}{\|\psi - \bar{\theta}\|}  \cdot \big|T_{\psi}(\mathbf{x}) - T_{\bar{\theta}}(\mathbf{x}) \big| 
        \cdot e^{\alpha    \tau_{\mathbf{x}}  |T_{\psi}(\mathbf{x}) - T_{\bar{\theta}}(\mathbf{x})|} 
        \cdot  e^{\alpha \cdot T_{\bar{\theta}}(\mathbf{x})}   \nonumber\\ 
        &\leq   L  \cdot e^{\alpha \cdot  L \cdot \| \psi - \bar{\theta} \|} \cdot  e^{\alpha \cdot T_{\bar{\theta}}(\mathbf{x})},  \label{proof_propos_grad_non_bias_to_get_upper_bound_q_1} 
    \end{align} 
    for all $\psi \in B_{\bar{\theta}}$. 

    Integrating the term on the left-hand side of Equation (\ref{proof_propos_grad_non_bias_to_get_upper_bound_q_1}) with respect to $Q$, we have 
    \begin{align} 
        &\int \left| 
        \frac{1}{\|\psi - \bar{\theta}\|} \cdot 
        \left\{ \frac{1}{\alpha} \cdot e^{\alpha \cdot T_{\psi}(\mathbf{x}^q)} - \frac{1}{\alpha} \cdot 
        e^{\alpha \cdot T_{\bar{\theta}}(\mathbf{x}^q)}  \right\} \right| 
        dQ(\mathbf{x}^q) \nonumber\\ 
        &\leq  \int   L  \cdot e^{\alpha \cdot  L \cdot \| \psi - \bar{\theta} \|} \cdot  e^{\alpha \cdot T_{\bar{\theta}}(\mathbf{x}^q)} dQ(\mathbf{x}^q) \nonumber\\ 
        &=   L  \cdot e^{\alpha \cdot L \cdot \| \psi - \bar{\theta} \|} \cdot E_Q\left[e^{\alpha \cdot T_{\bar{\theta}}} \right].    \label{proof_propos_grad_non_bias_to_get_upper_bound_q_2} 
    \end{align} 
    Considering the supremum for $\psi \in B_{\bar{\theta}}$ in Equation (\ref{proof_propos_grad_non_bias_to_get_upper_bound_q_2}), we obtain 
    \begin{align} 
        &\sup_{\psi \in B_{\bar{\theta}}} \Bigg\{ 
        \int \left|  \frac{1}{\|\psi - \bar{\theta}\|} \cdot 
        \left\{ \frac{1}{\alpha} \cdot e^{\alpha \cdot T_{\psi}} - \frac{1}{\alpha} e^{\alpha \cdot T_{\bar{\theta}}}  \right\} \right| dQ 
        \Bigg\} \nonumber\\ 
        &\leq \sup_{\psi \in B_{\bar{\theta}}} \bigg\{  L  \cdot e^{\alpha \cdot L \cdot \| \psi - \bar{\theta} \|} \,\cdot  E_Q\left[e^{\alpha \cdot T_{\bar{\theta}}} \right]  \bigg\} \nonumber\\ 
        &=  E_Q\left[e^{\alpha \cdot T_{\bar{\theta}}} \right]  \cdot \sup_{\psi \in B_{\bar{\theta}}} L  \cdot e^{\alpha \cdot  L \cdot \| \psi - \bar{\theta} \|} < \infty, \label{proof_propos_grad_non_bias_uniformaly_int_q} 
    \end{align} 
    since $B_{\bar{\theta}}$ is compact. 

    Therefore, the following set is uniformly integrable for $Q$: 
    \begin{equation} 
        \left\{ 
        \frac{1}{\|\psi - \bar{\theta}\|} 
        \left\{ 
        \frac{1}{\alpha} \cdot e^{\alpha \cdot T_{\psi}(\mathbf{x}^q)} - \frac{1}{\alpha} 
        e^{\alpha \cdot T_{\bar{\theta}}(\mathbf{x}^q)} 
        \right\} 
        \, : \, \psi \in B_{\bar{\theta}} \right\}. 
    \end{equation} 

    Similarly, for Equation (\ref{proof_propos_grad_non_bias_p}), we have 
    \begin{align} 
        &\sup_{\psi \in B_{\bar{\theta}}} \int \left| 
        \frac{1}{\|\psi - \bar{\theta}\|} \cdot  \left\{ \frac{1}{1-\alpha} 
        e^{(\alpha-1) \cdot T_{\psi}(\mathbf{x}^p)} - \frac{1}{1-\alpha} e^{(\alpha-1) \cdot T_{\bar{\theta}}(\mathbf{x}^p)}  \right\} \right| 
        dP(\mathbf{x}^p) \nonumber\\ 
        &\leq \sup_{\psi \in B_{\bar{\theta}}} \Big\{  L  \cdot e^{(1-\alpha)  L \cdot \| \psi - \bar{\theta} \|} \, \cdot  E_P \left[e^{(1-\alpha) \cdot T_{\bar{\theta}}} \right]  \Big\} \nonumber\\ 
        &=  E_P \left[e^{(1-\alpha) \cdot T_{\bar{\theta}}} \right]  \cdot \sup_{\psi \in B_{\bar{\theta}}} L \cdot e^{(1-\alpha)  L \cdot \| \psi - \bar{\theta} \|} < \infty. 
        \label{proof_propos_grad_non_bias_uniformaly_int_p} 
    \end{align} 

    Therefore, the following set is uniformly integrable for $P$: 
    \begin{equation} 
        \left\{ 
        \frac{1}{\|\psi - \bar{\theta}\|}\cdot \left\{ 
        \frac{1}{1-\alpha} \cdot e^{(\alpha-1) \cdot T_{\psi}(\mathbf{x}^p)} - \frac{1}{1-\alpha} \cdot e^{(\alpha-1) \cdot T_{\bar{\theta}}(\mathbf{x}^p)} 
        \right\} 
        \, : \, \psi \in B_{\bar{\theta}} \right\}. 
        \label{Eq_Appendixtheorem_gradient_unbiased_1} 
    \end{equation} 
    Thus, the Lebesgue integral and $\lim_{\psi \rightarrow \bar{\theta}}$ are exchangeable for the set in Equation (\ref{Eq_Appendixtheorem_gradient_unbiased_1}). Then, we have 
    \begin{align} 
        &\nabla_{\theta} E_Q \left[ 
        \frac{1}{\alpha} \cdot  e^{\alpha \cdot T_{\theta}(\mathbf{x}^q)} 
        \right] \Bigg|_{\theta = \bar{\theta}} \nonumber\\ 
        &= \lim_{\psi \rightarrow \bar{\theta}} \int  \frac{1}{\|\psi - \bar{\theta}\|}\cdot 
        \bigg\{ 
        \frac{1}{\alpha} \cdot e^{\alpha \cdot T_{\psi}(\mathbf{x}^q)} - \frac{1}{\alpha}\cdot e^{\alpha \cdot T_{\bar{\theta}}(\mathbf{x}^q)} 
        \bigg\}  dQ(\mathbf{x}^q) \nonumber\\ 
        &=   \int \lim_{\psi \rightarrow \bar{\theta}} \left[ 
        \frac{1}{\|\psi - \bar{\theta}\|} \cdot 
        \bigg\{ 
        \frac{1}{\alpha} \cdot e^{\alpha \cdot T_{\psi}(\mathbf{x}^q)} 
         - \frac{1}{\alpha} \cdot e^{\alpha \cdot T_{\bar{\theta}}(\mathbf{x}^q)} 
        \bigg\} \right] 
        dQ(\mathbf{x}^q)  \nonumber\\ 
        &=  E_Q \left[ 
        \nabla_{\theta} 
        \left( 
        \frac{1}{\alpha} \cdot e^{\alpha \cdot T_{\theta}(\mathbf{x}^q)} 
        \right) 
        \Bigg|_{\theta = \bar{\theta}} 
        \right].  \label{proof_propos_grad_non_bias_excahgeable_lim_int_Q} 
    \end{align}

    Similarly, we obtain 
    \begin{align} 
        &\nabla_{\theta} E_P \left[ 
        \frac{1}{1- \alpha} \cdot e^{(\alpha-1) \cdot T_{\theta}(\mathbf{x}^p)} 
        \right] \Bigg|_{\theta = \bar{\theta}} \nonumber\\ 
        &= \lim_{\psi \rightarrow \bar{\theta}} \int  \frac{1}{\|\psi - \bar{\theta}\|} 
        \cdot 
        \bigg\{ 
        \frac{1}{1- \alpha} \cdot e^{(\alpha-1) \cdot T_{\psi}(\mathbf{x}^p)} - \frac{1}{1- \alpha}\cdot e^{(\alpha-1) \cdot T_{\bar{\theta}}(\mathbf{x}^p)} 
        \bigg\}  dP(\mathbf{x}^p) \nonumber\\ 
        &=   \int \lim_{\psi \rightarrow \bar{\theta}} \left[ 
        \frac{1}{\|\psi - \bar{\theta}\|} \cdot 
        \bigg\{ 
        \frac{1}{1- \alpha} \cdot e^{(\alpha-1) \cdot T_{\psi}(\mathbf{x}^p)} - \frac{1}{1- \alpha} \cdot e^{(\alpha-1) \cdot T_{\bar{\theta}}(\mathbf{x}^p)} 
        \bigg\} \right] 
        dP(\mathbf{x}^p) \nonumber\\ 
        &= E_P \left[ \nabla_{\theta} 
        \left( 
        \frac{1}{1- \alpha}\cdot  e^{(\alpha-1) \cdot T_{\theta}(\mathbf{x}^p)} 
        \right) 
        \Bigg|_{\theta = \bar{\theta}} 
        \right].  \label{proof_propos_grad_non_bias_excahgeable_lim_int_P} 
    \end{align} 

    From Equations (\ref{proof_propos_grad_non_bias_excahgeable_lim_int_Q}) and (\ref{proof_propos_grad_non_bias_excahgeable_lim_int_P}), we have 
    \begin{align} 
        &E\left[ \nabla_{\theta} \, \mathcal{L}_{\alpha\text{-Div}}^{(R,S)}(T_{\theta}\,;\, \alpha)\Big|_{\theta = \bar{\theta}} \right]\nonumber\\ 
        &= E_P \left[ E_Q \left[ \nabla_{\theta} \, \mathcal{L}_{\alpha\text{-Div}}^{(R,S)}(T_{\theta}\,;\, \alpha)\Big|_{\theta = \bar{\theta}} \right] \right] \nonumber\\ 
        &= E_P \left[ E_Q \left[ \nabla_{\theta}|_{\theta = \bar{\theta}} 
        \left\{ 
        \frac{1}{\alpha} \cdot  \frac{1}{S}  \cdot \sum_{i=1}^{S} e^{\alpha\cdot T_{\theta}(\mathbf{x}_{i}^{q})} + \frac{1}{1- \alpha} \cdot  \frac{1}{R} \cdot \sum_{i=1}^{R} e^{(\alpha -1 )\cdot T_{\theta}(\mathbf{x}_{i}^{p})} 
        \right\} \right] \right] \nonumber\\ 
        &= E_P \left[ E_Q \left[ 
        \frac{1}{\alpha} \cdot  \frac{1}{S}  \cdot \sum_{i=1}^{S} \nabla_{\theta} \left( e^{\alpha\cdot T_{\theta}(\mathbf{x}_{i}^{p})}\right)\Big|_{\theta = \bar{\theta}} 
        \right. \right. \nonumber\\ 
        &\left. \left. \qquad \qquad \qquad + \ 
        \frac{1}{1- \alpha} \cdot  \frac{1}{R} \cdot \sum_{i=1}^{R} \nabla_{\theta} \left(e^{(\alpha -1 )\cdot T_{\theta}(\mathbf{x}_{i}^{p})}  \right) \Big|_{\theta = \bar{\theta}}\, 
        \right] \right] \nonumber\\ 
        &= \frac{1}{\alpha} \cdot  \frac{1}{S}  \cdot \sum_{i=1}^{S}  E_Q \left[\nabla_{\theta} 
        \left(e^{\alpha\cdot T_{\theta}(\mathbf{x}_{i}^q)} 
        \right) \right]\Big|_{\theta = \bar{\theta}} \nonumber\\ 
        &\qquad \qquad \qquad + \ 
        \frac{1}{1- \alpha} \cdot  \frac{1}{R} \cdot \sum_{i=1}^{R}  E_P \left[\nabla_{\theta}\left( 
        e^{(\alpha -1 )\cdot T_{\theta}(\mathbf{x}_{i}^p)} 
        \right) \Big|_{\theta = \bar{\theta}}\right] \nonumber\\ 
        &= \frac{1}{\alpha} \cdot  \frac{1}{S}  \cdot \sum_{i=1}^{S} \nabla_{\theta} E_Q \left[ e^{\alpha\cdot T_{\theta}(\mathbf{x}_{i}^q)} \right] \Big|_{\theta = \bar{\theta}} \nonumber\\ 
        &\qquad \qquad \qquad+ \ 
        \frac{1}{1- \alpha} \cdot  \frac{1}{R} \cdot \sum_{i=1}^{R} \nabla_{\theta} E_P \left[ e^{(\alpha -1 )\cdot T_{\theta}(\mathbf{x}_{i}^p)}\right] \Big|_{\theta = \bar{\theta}} \nonumber\\ 
        &= \nabla_{\theta}\, E_Q\left[ 
        \frac{1}{\alpha} \cdot  \frac{1}{S}  \cdot \sum_{i=1}^{S}   e^{\alpha\cdot T_{\theta}(\mathbf{x}_{i}^q)} \right] \Big|_{\theta = \bar{\theta}} \nonumber\\ 
        &\qquad \qquad  \qquad+ \ 
        \nabla_{\theta}\, E_P\left[ 
        \frac{1}{1- \alpha} \cdot  \frac{1}{R} \cdot \sum_{i=1}^{R} 
        e^{(\alpha -1 )\cdot T_{\theta}(\mathbf{x}_{i}^p)}\right] \Big|_{\theta = \bar{\theta}} \nonumber\\ 
        &= \nabla_{\theta} \,\left\{ 
        E_P \left[ E_Q \left[ 
        \frac{1}{\alpha} \cdot  \frac{1}{S}  \cdot \sum_{i=1}^{S} e^{\alpha\cdot T_{\theta}(\mathbf{x}_{i}^{q})} 
        \right. \right. \right. \nonumber\\ 
        &\left. \left. \left. \qquad \qquad + \ 
        \frac{1}{1- \alpha} \cdot  \frac{1}{R} \cdot \sum_{i=1}^{R} e^{(\alpha -1 )\cdot T_{\theta}(\mathbf{x}_{i}^{p})} 
        \right] \right] \right\} \Big|_{\theta = \bar{\theta}} \nonumber\\ 
        &= \nabla_{\theta} \,E_P \left[ E_Q \left[ \mathcal{L}_{\alpha\text{-Div}}^{(R,S)}(T_{\theta}\,;\, \alpha) \right]  \right] \Big|_{\theta = \bar{\theta}} \nonumber\\ 
        &= \nabla_{\theta} \, E\left[ \mathcal{L}_{\alpha\text{-Div}}^{(R,S)}(T_{\theta}\,;\, \alpha) \right] \Big|_{\theta = \bar{\theta}}. 
    \end{align} 

    This completes the proof. 
\end{proof} 


\begin{theorem}\label{Appendix_theorem_consistency_alpha_div_est} 
    Assume $E_P\big[(dQ/dP(\mathbf{X}))^{2\cdot \alpha}\big] < \infty$. 
    Let $T^*=- \log dQ/dP$. Subsequently, let 
    \begin{equation} 
        \hat{D}^{(N)} (Q||P\,;\, \alpha) = \frac{1}{\alpha\cdot(1-\alpha)} - \mathcal{L}_{\alpha\text{-Div}}^{(N,N)}(T^*\,;\, \alpha). 
    \end{equation} 
    Then, it holds that as $N \rightarrow \infty$, 
    \begin{equation} 
        \sqrt{N} \left\{ \hat{D}^{(N)} (Q||P\,;\, \alpha) - D(Q||P\, ;\, \alpha) \right\} \quad \xrightarrow{\ \  d \ \ }  \quad \mathcal{N}  \big(0, \ \sigma_{\alpha}^2 \big), 
    \end{equation} 
    where 
    \begin{align} 
        \sigma_{\alpha}^2&= C^1_{\alpha} \cdot D(Q||P\,;\, 2 \alpha) + C^2_{\alpha} \cdot  D(Q||P\,;\, 2 \alpha - 1) \nonumber\\ 
        & \qquad \qquad + \  C^3_{\alpha} \cdot  D(Q||P\, ;\, \alpha)^2 +  C^4_{\alpha} \cdot  D(Q||P\, ;\, \alpha) + C^5_{\alpha}, \label{eq_variance_alpha_not_half} 
    \end{align} 
    and 
    \begin{align} 
        C^1_{\alpha} &=  \frac{ 2\alpha \cdot (1 - 2\alpha)}{\alpha^2}, \\ 
        C^2_{\alpha} &=  \frac{ 2\alpha \cdot (1 - 2\alpha)}{(1-\alpha)^2}, \\ 
        C^3_{\alpha} &=  - \frac{1}{\alpha^2} - \frac{1}{(1 - \alpha)^2}, \\ 
        C^4_{\alpha} &=  \frac{2}{\alpha^2} + \frac{2}{(1 - \alpha)^2}, \nonumber \\ 
        C^5_{\alpha} &=  \left(\frac{1}{\alpha^2} + \frac{1}{(1 - \alpha)^2}\right) \cdot (2 - 2 \alpha \cdot (1-\alpha)). 
    \end{align} 
\end{theorem} 
\begin{proof}[proof of Theorem \ref{Appendix_theorem_consistency_alpha_div_est}] 
    First, note that 
    \begin{align} 
        \hat{D}^{(N)} (Q||P\,;\, \alpha) &= \frac{1}{\alpha\cdot(1-\alpha)} 	- \frac{1}{\alpha} \cdot \left\{ \frac{1}{N}\cdot \sum_{i=1}^N e^{\alpha \cdot T^*(\mathbf{X}^{i}_{Q} )}  \right\} 
        -  \frac{1}{1- \alpha} \cdot \left\{  \frac{1}{N} \cdot 
        \sum_{i=1}^N  e^{(\alpha - 1) \cdot T^*(\mathbf{X}^{i}_{P})} \right\} \nonumber\\ 
        &= \frac{1}{\alpha\cdot(1-\alpha)} 
        - \frac{1}{\alpha} \cdot \left\{ \frac{1}{N} \cdot \sum_{i=1}^N \left( \frac{dQ}{dP} (\mathbf{X}^{i}_{Q})  \right)^{-\alpha} \right\}\nonumber\\ 
        & \qquad \qquad \qquad\qquad\qquad- \  \frac{1}{1- \alpha} \cdot \left\{ 
         \frac{1}{N} \cdot \sum_{i=1}^N 
        \left( \frac{dQ}{dP}(\mathbf{X}^{i}_{P}) \right)^{1-\alpha}  \right\}. \label{D_alpha_hat_represented_as_sum_iids} 
    \end{align} 

    On the other hand, from Lemma \ref{lemma_loss_not_biased_lemma}, we obtain 
    \begin{align} 
        D(Q||P\, ;\, \alpha) &=  \sup_{T:\Omega \rightarrow \mathbb{R}} \left\{ 
        \frac{1}{\alpha\cdot(1-\alpha)} - \frac{1}{\alpha} \cdot E_{Q} \left[e^{\alpha \cdot T} \right] 
        - \frac{1}{1- \alpha} \cdot E_{P} \cdot \left[ e^{(\alpha - 1) \cdot T} \right] 	\right\} \nonumber\\ 
        &=  \frac{1}{\alpha\cdot(1-\alpha)} - 
        \inf_{T:\Omega \rightarrow \mathbb{R}} \left\{ \frac{1}{\alpha} \cdot E_{Q} \left[e^{\alpha \cdot T} \right] 
        - \frac{1}{1- \alpha} \cdot E_{P} \left[ e^{(\alpha - 1) \cdot T} \right] \right\} \nonumber\\ 
        &= \frac{1}{\alpha\cdot(1-\alpha)}  -  \widebar{\mathcal{L}}_{\alpha\text{-Div}}^* (\alpha) \nonumber\\ 
        &= \frac{1}{\alpha\cdot(1-\alpha)} -  \frac{1}{\alpha}\cdot  E_{Q} \left[e^{\alpha \cdot T^*} \right] 
        - \frac{1}{1- \alpha} \cdot  E_{P} \left[ e^{(\alpha - 1) \cdot T^*} \right]  \nonumber\\ 
        &= 
        \frac{1}{\alpha\cdot(1-\alpha)} 
        - \frac{1}{\alpha} \cdot E_Q\left[ \left( \frac{dQ}{dP} (\mathbf{x})  \right)^{-\alpha} \right] 
        -  \frac{1}{1- \alpha} \cdot E_P\left[ \left( \frac{dQ}{dP}(\mathbf{x}) \right)^{1-\alpha} \right]  \nonumber\\ 
        &=  \frac{1}{\alpha\cdot(1-\alpha)} 
        - \frac{1}{\alpha}  \cdot \left\{ \frac{1}{N}\cdot  \sum_{i=1}^N E_Q \left[ \left( \frac{dQ}{dP} (\mathbf{x}_{i} )  \right)^{-\alpha}  \right] \right\} \nonumber\\ 
        & \qquad \qquad \qquad \qquad\qquad - \  \frac{1}{1- \alpha}\cdot \left\{ \frac{1}{N}\cdot \sum_{i=1}^N E_P \left[ \left( \frac{dQ}{dP}(\mathbf{x}_{i} ) \right)^{1-\alpha} \right] \right\}. \nonumber\\ 
        \label{D_alpha_represented_as_exp} 
    \end{align}

    Subtracting Equation (\ref{D_alpha_represented_as_exp}) from Equation (\ref{D_alpha_hat_represented_as_sum_iids}), we have 
    \begin{align} 
        &\hat{D}^{(N)} (Q||P\,;\, \alpha) -  D(Q||P\, ;\, \alpha)\nonumber\\ 
        &= \frac{1}{N}\cdot \sum_{i=1}^N 
        \frac{1}{\alpha}  \cdot  \left\{  \left( \frac{dQ}{dP} (\mathbf{X}^{i}_{Q}) \right)^{- \alpha} - E_Q \left[ \left( \frac{dQ}{dP} (\mathbf{x}_i)  \right)^{-\alpha}  \right] \right\} 
        \nonumber\\ 
        & \qquad +  \   \frac{1}{N} \cdot \sum_{i=1}^N 
        \frac{1}{1- \alpha} \cdot \left\{ 
        \left( \frac{dQ}{dP}(\mathbf{X}^{i}_{P}) \right)^{1-\alpha}  -  E_P \left[ \left( \frac{dQ}{dP}(\mathbf{x}_i) \right)^{1-\alpha} \right] \right\} \nonumber\\ 
        &= \frac{1}{N}\cdot \sum_{i=1}^N 
        \frac{1}{\alpha}  \cdot  \left\{  \left( \frac{dQ}{dP} (\mathbf{X}^{i}_{Q}) \right)^{- \alpha} - E_Q \left[ \left( \frac{dQ}{dP} (\mathbf{x})  \right)^{-\alpha}  \right] \right\} 
        \nonumber\\ 
        & \qquad +  \   \frac{1}{N} \cdot \sum_{i=1}^N 
        \frac{1}{1- \alpha} \cdot \left\{ 
        \left( \frac{dQ}{dP}(\mathbf{X}^{i}_{P}) \right)^{1-\alpha}  -  E_P \left[ \left( \frac{dQ}{dP}(\mathbf{x}) \right)^{1-\alpha} \right] \right\}. 
    \end{align} 

    Let $L_{Q}^i = \frac{1}{\alpha} \cdot \left\{ \left( \frac{dQ}{dP} (\mathbf{X}^{i}_{Q}) \right)^{-\alpha} 
    - E_Q \left[ \left( \frac{dQ}{dP} (\mathbf{x})  \right)^{-\alpha} \right]\right\} $. 
    Then $\{L_{Q}^i\}_{i=1}^N$ are independent and identically distributed variables whose means and variances are as follows: 
    \begin{align} 
        E_Q \Big[L_{Q}^i \Big] &= 0, 
    \end{align} 
    and 
    \begin{align} 
        &\mathrm{Var}_Q \Big[ L_{Q}^i \Big] \nonumber\\ 
        &= E_Q \left[ 
        \frac{1}{\alpha^2} \cdot 
        \left\{ 
        \ \left( \frac{dQ}{dP} (\mathbf{x}) \right)^{-\alpha} - 
        E_Q \left[ \left( \frac{dQ}{dP} (\mathbf{x})  \right)^{-\alpha} \right] 
        \right\}^2 
        \right] \nonumber\\ 
        &= \frac{1}{\alpha^2} \cdot E_{Q} \left[  \left\{ 
        \left( \frac{dQ}{dP}  (\mathbf{x} )\right)^{-\alpha} 
        \right\}^2 \ 
        \right] 
        - 
        \frac{1}{\alpha^2} \cdot 
        \left\{ 
        E_{Q} 
        \left[ \left( \frac{dQ}{dP}  (\mathbf{x})\right)^{-\alpha} \ \right] 
        \right\}^2 \nonumber \\ 
        &= \frac{1}{\alpha^2} \cdot E_{P} \left[ 
        \frac{dQ}{dP}(\mathbf{x}) 
        \cdot 
        \left( \frac{dQ}{dP}  (\mathbf{x} )\right)^{-2 \alpha} 
        \ \right] 
        - 
        \frac{1}{\alpha^2} \cdot 
        \left\{ 
        E_{P} \left[ 
        \frac{dQ}{dP} (\mathbf{x}) \cdot \left( \frac{dQ}{dP} (\mathbf{x})  \right)^{-\alpha}\ \right] 
        \right\}^2 
        \nonumber \\ 
        &= \frac{1}{\alpha^2} 
        \cdot E_{P} \left[  \left(\frac{dQ}{dP}  (\mathbf{x}) \right)^{1 - 2\alpha} \right] - \frac{1}{\alpha^2} 
        \cdot \left\{E_{P} \left[\left(\frac{dQ}{dP} (\mathbf{x}) \right)^{1 - \alpha} \ \right] \right\}^2 \nonumber \\ 
        &= \frac{1}{\alpha^2} 
        \cdot   \Bigg\{  \  2\alpha \cdot (2\alpha - 1) \cdot \left( \frac{1}{ 2\alpha \cdot (2\alpha - 1)}\cdot 
        E_{P} \left[ \left(\frac{dQ}{dP}(\mathbf{x}) \right)^{1- 2\alpha}  - 1 \right] \right)  \  + 1 \nonumber\\ 
        & \qquad \qquad - \  \alpha^2 \cdot (1 - \alpha)^2 \cdot \left( 
        \frac{1}{\alpha\cdot (\alpha - 1)}\cdot 
        E_{P} \left[ \left(\frac{dQ}{dP}(\mathbf{x}) \right)^{1 - \alpha}  - 1 \  \right] 
        \right)^2 + 1\nonumber\\ 
        & \qquad \qquad + \  \alpha^2 \cdot (1 - \alpha)^2 \cdot \left( 
        \frac{2}{\alpha\cdot (\alpha - 1)}\cdot 
        E_{P} \left[ \left(\frac{dQ}{dP} (\mathbf{x})\right)^{1 - \alpha}  - 1 \  \right] \right) \nonumber\\ 
        & \qquad \qquad  \qquad  - \  2 \alpha \cdot (1-\alpha) \  \Bigg\}.     \label{Eq_var_LQ} 
    \end{align}

      Similarly, let $L_{P}^i = \frac{1}{1 - \alpha} \cdot \left\{ \left( \frac{dQ}{dP} (\mathbf{X}^{i}_{P}) \right)^{1-\alpha} - E_P \left[ \left( \frac{dQ}{dP} (\mathbf{x})  \right)^{1-\alpha} \right]\right\} $. 
      Then $\{L_{P}^i\}_{i=1}^N$ are independent and identically distributed variables whose means and variances are as follows: 
      \begin{align} 
          E_P \Big[L_{P}^i \Big] &= 0, 
      \end{align} 
      and 
      \begin{align} 
          &\mathrm{Var}_P \Big[ L_{P}^i \Big]\nonumber\\ 
          &= E_P \left[ 
          \frac{1}{(1 - \alpha)^2} \cdot 
          \left\{ 
          \ \left( \frac{dQ}{dP} (\mathbf{x}) \right)^{1- \alpha} - 
          E_P \left[ \left( \frac{dQ}{dP} (\mathbf{x})  \right)^{1 - \alpha} \right] 
          \right\}^2 
          \right] \nonumber\\ 
          &= \frac{1}{(1- \alpha)^2} \cdot E_{P} \left[  \left\{ 
          \left( \frac{dQ}{dP} (\mathbf{x}) \right)^{1 - \alpha} 
          \right\}^2 \ 
          \right] 
          - 
          \frac{1}{(1 -\alpha)^2} \cdot 
          \left\{ 
          E_{P} 
          \left[ \left( \frac{dQ}{dP} (\mathbf{x})  \right)^{1-\alpha} \ \right] 
          \right\}^2 \nonumber \\ 
          &= \frac{1}{(1- \alpha)^2} \cdot E_{P} \left[ 
          \left( \frac{dQ}{dP}  (\mathbf{x} ) \right)^{2(1 - \alpha)} 
          \ \right] 
          - 
          \frac{1}{(1 - \alpha)^2} \cdot 
          \left\{ 
          E_{P} \left[ \left( \frac{dQ}{dP}  (\mathbf{x}) \right)^{1 -\alpha} \ \right] 
          \right\}^2 
          \nonumber \\ 
          &= \frac{1}{(1- \alpha)^2} 
          \cdot E_{P} \left[  \left(\frac{dQ}{dP} (\mathbf{x}) \right)^{1 - (2\alpha - 1)}  \right] - \frac{1}{(1 - \alpha)^2} 
          \cdot \left\{E_{P} \left[\left(\frac{dQ}{dP}(\mathbf{x})  \right)^{1 - \alpha} \ \right] \right\}^2 \nonumber \\ 
          &= \frac{1}{(1 - \alpha)^2} 
          \cdot   \Bigg\{  \  2\alpha \cdot (2\alpha - 1) \cdot \left( \frac{1}{ 2\alpha \cdot (2\alpha - 1)}\cdot 
          E_{P} \left[ \left(\frac{dQ}{dP}(\mathbf{x}) \right)^{1- (2\alpha- 1)}  - 1 \right] \right)  \  + 1 \nonumber\\ 
          & \qquad \qquad - \  \alpha^2 \cdot (1 - \alpha)^2 \cdot \left( 
          \frac{1}{\alpha\cdot (\alpha - 1)}\cdot 
          E_{P} \left[ \left(\frac{dQ}{dP}(\mathbf{x}) \right)^{1 - \alpha}  - 1 \  \right] 
          \right)^2 + 1\nonumber\\ 
          & \qquad \qquad + \ \alpha^2 \cdot (1 - \alpha)^2 \cdot \left( 
          \frac{2}{\alpha\cdot (\alpha - 1)}\cdot 
          E_{P} \left[ \left(\frac{dQ}{dP}(\mathbf{x}) \right)^{1 - \alpha}  - 1 \  \right] \right) \nonumber\\ 
          &  \qquad \qquad  \qquad - \  2 \alpha \cdot (1-\alpha) \  \Bigg\}.     \label{Eq_var_LP} 
      \end{align} 

    Now, we consider an asymptotical distribution of the following term: 
    \begin{align} 
        & \sqrt{N} \cdot \left\{ \hat{D}^{(N)} (Q||P\,;\, \alpha) -  D(Q||P\, ;\, \alpha)  \right\}  \nonumber\\ 
        &=  \sqrt{N} \cdot \left(   \frac{1}{N} \cdot \sum_{i=1}^{N} \Big\{ L_{Q}^i - E_Q \left[ L_{Q}^i \right]\Big\} \right) 
        +  \sqrt{N} \cdot \left(   \frac{1}{N} \cdot \sum_{i=1}^{N} \Big\{  L_{P}^i -  E_P \left[ L_{P}^i \right]   \Big\}  \right). 
    \end{align} 

    By the central limit theorem, we observe that as $N \rightarrow \infty$, 
    \begin{equation} 
        \sqrt{N} \cdot  \left(   \frac{1}{N} \cdot  \sum_{i=1}^{N} \Big\{ L_{Q}^i - E_Q \left[ L_{Q}^i \right]\Big\} \right)   \quad  \xrightarrow{\ \  d \ \ }  \quad  \mathcal{N}\big(0, \mathrm{Var}_Q \left[ L_{Q}^i \right]\big), 
        \label{Eq_c_lim_1} 
    \end{equation} 
    and 
    \begin{equation} 
        \sqrt{N} \cdot  \left(   \frac{1}{N} \cdot  \sum_{i=1}^{N} \Big\{ L_{P}^i - E_P \left[ L_{P}^i \right]\Big\} \right) \quad  \xrightarrow{\ \  d \ \ }  \quad  \mathcal{N}\big(0, \mathrm{Var}_P \left[ L_{P}^i \right]\big). 
        \label{Eq_c_lim_2} 
    \end{equation} 

    Therefore, from Equations (\ref{Eq_c_lim_1}) and (\ref{Eq_c_lim_2}), we obtain 
    \begin{equation} 
        \sqrt{N} \cdot  \left\{ \hat{D}^{(N)} (Q||P\,;\, \alpha) -  D(Q||P\, ;\, \alpha)  \right\}   \xrightarrow{\ \  d \ \ } \mathcal{N}\big(0, \sigma_{\alpha}^2\big), 
    \end{equation} 
    and 
    \begin{align} 
        \sigma_{\alpha}^2 &= \mathrm{Var}_Q \left[L_Q^i  \right] + \mathrm{Var}_P \left[L_P^i  \right] \nonumber \\ 
        &= C^1_{\alpha} \cdot D(Q||P\,;\, 2\alpha) + C^2_{\alpha} \cdot  D(Q||P\,;\, 2 \alpha - 1) \nonumber\\ 
        & \qquad \quad + \  C^3_{\alpha} \cdot  D(Q||P\, ;\, \alpha)^2 +  C^4_{\alpha} \cdot  D(Q||P\, ;\, \alpha) + C^5_{\alpha}, \label{eq_variance_alpha_not_half2} 
    \end{align} 
    where 
    \begin{align} 
        C^1_{\alpha} &=  \frac{ 2\alpha \cdot (1 - 2\alpha)}{\alpha^2}, \\ 
        C^2_{\alpha} &=  \frac{ 2\alpha \cdot (1 - 2\alpha)}{(1-\alpha)^2}, \\ 
        C^3_{\alpha} &=  - \frac{1}{\alpha^2} - \frac{1}{(1 - \alpha)^2}, \\ 
        C^4_{\alpha} &=  \frac{2}{\alpha^2} + \frac{2}{(1 - \alpha)^2}, \nonumber \\ 
        C^5_{\alpha} &=   \left(\frac{1}{\alpha^2} + \frac{1}{(1 - \alpha)^2}\right) \cdot \big(2 - 2 \alpha \cdot (1-\alpha)\big). 
    \end{align} 

    This completes the proof. 
\end{proof}

\section{Details of the experiments in Section \ref{Section_NumericalExperiment}}\label{Section_Appendix_TheDetailsOfNumericalExperiments} 
In this section, we provide details of the hyperparameter settings used in the experiments described in Section \ref{Section_NumericalExperiment}. 


\subsection{Details of the experiments in Section \ref{subsection_Experimentsonthestabilityinoptimizationfordifferentvaluesofalpha}} \label{Apdx_Section_thedetailsExperimentsonthestabilityinoptimizationfordifferentvaluesofalpha} 
In this section, we provide details of the experiments reported in Section \ref{subsection_Experimentsonthestabilityinoptimizationfordifferentvaluesofalpha}. 

\subsubsection{Datasets.} 
We generated the following 100 train datasets. 
$P=\mathcal{N}(\mu_p, I_5)$ and $Q=\mathcal{N}(\mu_q, \Sigma_q)$ where 
$\mu_p=\mu_q=(0,0,\ldots,0)$, and $I_5$ denotes the $5$-dimensional identity matrix, and 
$\Sigma_q=(\sigma_{ij})_{i=1}^5$ with $\sigma_{ii}=1$, and $\sigma_{ij}=0.8$ for $i\neq j$. 
The size of each dataset was 5000. 

\subsubsection{Experimental Procedure.} 
Neural networks were trained using the synthetic datasets by optimizing $\alpha$-Div for  $\alpha=-3.0,\allowbreak -2.0, \allowbreak-1.0,\allowbreak 
0.2,\allowbreak 0.5,\allowbreak 0.8,\allowbreak 
2.0,\allowbreak 3.0$, and $4.0$ while measuring the training losses for each learning step. 
For each value of $\alpha$, 100 trials were conducted. 
Finally, we reported the median, ranging between the 45th and 55th quartiles, and between the 2.5th and 97.5th quartiles of the training losses at each learning step. 

\subsubsection{Neural Network Architecture, Optimization Algorithm, and Hyperparameters.} 
A 5-layer perceptron with ReLU activation was used, with each hidden layer comprising 100 nodes. For optimization, the learning rate was set to 0.001, the batch size to 2500, and the number of epochs to 250. 
The models for DRE were implemented using the PyTorch library \citep{paszke2017automatic} in Python. Training was conducted with the Adam optimizer \citep{kingma2014adam} in PyTorch and an NVIDIA T4 GPU.

\subsubsection{Results.} 
Figure \ref{Ap_Figure_subsection_Experimentsonthestabilityinoptimizationfordifferentvaluesofalpha_all} presents the training losses of $\alpha$-Div across learning steps for $\alpha=-3, -2, -1, 0.2, 0.5, 0.8, 2.0, 3.0$, and $4.0$. 
The upper ($\alpha=-3.0, -2.0$, and $-1.0$) and middle ($\alpha=2.0, 3.0$ and $4.0$) figures in Figure \ref{Ap_Figure_subsection_Experimentsonthestabilityinoptimizationfordifferentvaluesofalpha_all}show that the training losses diverged to large negative values when $\alpha < 0$ or $\alpha > 1$. 
In contrast, the bottom figure ($\alpha = 0.2, 0.5$, and $0.8$) Figure \ref{Figure_subsection_Experimentsonthestabilityinoptimizationfordifferentvaluesofalpha}, the training losses of $\alpha$-Div converged, 
 illustrating the stability of optimization with $\alpha$-Div when $0 < \alpha < 1$.

\begin{figure*}[t] 
    \begin{center} 
        \centerline{\includegraphics[width=1.00\columnwidth]{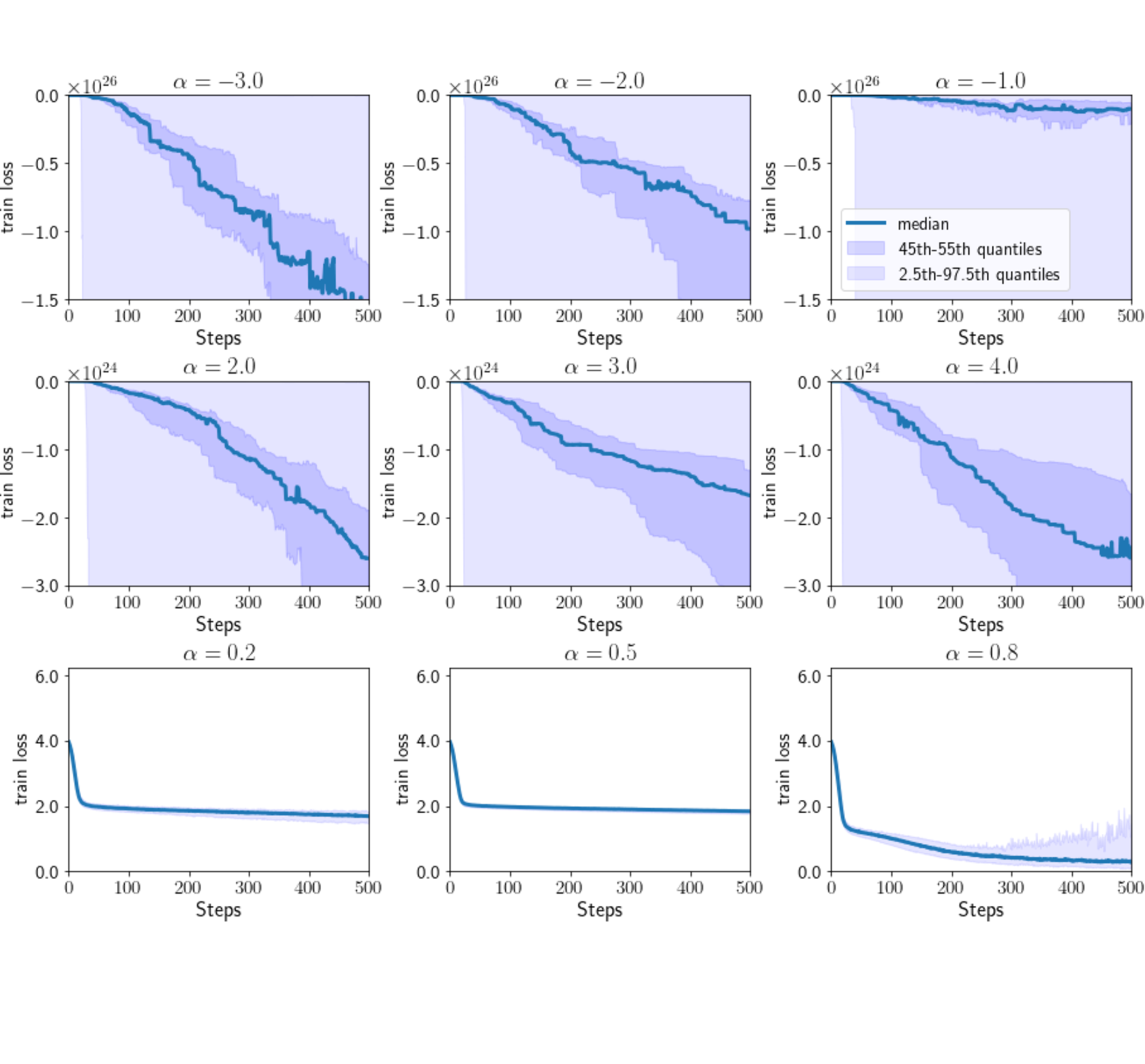}} 
        \caption{ 
            All results of Section \ref{subsection_Experimentsonthestabilityinoptimizationfordifferentvaluesofalpha}. 
            for $\alpha=-3, -2, -1, 0.2, 0.5, 0.8, 2.0, 3.0$, and $4.0$. 
      Each graph displays the training losses ($y$-axis) against the learning steps ($x$-axis) during optimization using $\alpha$-Div for different values of $\alpha$. 
      The solid blue line represents the median training losses. 
      The dark blue area indicates the range between the 45th and 55th percentiles, while the light blue area shows the range between the 2.5th and 97.5th percentiles of the training losses.} 
        \label{Ap_Figure_subsection_Experimentsonthestabilityinoptimizationfordifferentvaluesofalpha_all} 
    \end{center} 
\end{figure*}

\subsection{ 
    Details of the experiments in Section \ref{subsection_Experimentsonimprovementofoptimizationefficiencybyremovinggradientbias}} 
\label{Apdx_Subection_subsection_Experimentsonimprovementofoptimizationefficiencybyremoving} 
In this section, we provide details about the experiments reported in Section \ref{subsection_Experimentsonimprovementofoptimizationefficiencybyremovinggradientbias}. 

\subsubsection{Datasets.} 
We first generated 100 training datasets, each with a total size of $10000$ samples. 
Each dataset was drawn from two normal distributions: $P=\mathcal{N}(\mu_p, \cdot I_5)$ and $Q=\mathcal{N}(\mu_q, \cdot I_5)$ where $I_5$ denotes the $5$-dimensional identity matrix, and $\mu_p=(-5/2, 0, 0, 0, 0, 0)$ and $\mu_q=(5/2, 0, 0, 0, 0, 0)$. 

\subsubsection{Experimental Procedure.} 
We trained neural networks using the training datasets, optimizing both $\alpha$-Div and the standard $\alpha$-divergence loss function defined in Equation (\ref{Eq_loss_func_standard_alpha_div}) with $\alpha=0.5$, as well as nnBD-LSIF, while measuring the training losses at each learning step. 
We conducted 100 trials and reported the median training losses, along with the ranges between the 45th and 55th percentiles, and between the 2.5th and 97.5th percentiles, at each learning step. 

\paragraph{Loss functions used in the experiments.} 
We used $\alpha$-Div, the standard $\alpha$-divergence loss function, and the non-negative Bregman divergence least-squares importance fitting (nnBD-LSIF) loss function \citep{kato2021non} to train neural networks. 
The standard $\alpha$-divergence loss function, presented in Equation (\ref{Eq_loss_func_standard_alpha_div}), exhibits a biased gradient when $\alpha < 1$. 

nnBD-LSIF is an unbounded Bregman divergence loss function obtained from the deep direct DRE (D3RE) method proposed by \citet{kato2021non}, which is defined as 
\begin{equation} 
    \mathcal{L}_{\text{$\mathrm{nnBD}$-$\mathrm{LSIF}$}} (\phi) = 
    - \hat{E}_Q\left[\phi(\mathbf{x})- \frac{C}{2}\phi^2(\mathbf{x})\right] 
    + \left( \frac{1}{2} \cdot \hat{E}_P\left[ \phi^2(\mathbf{x}) \right]  -  \frac{C}{2} \cdot \hat{E}_Q\left[ \phi^2(\mathbf{x})\right]  \right)_{+}, 
\end{equation} 
where $(a)_{+}=a$ if $a > 0$ otherwise $(a)_{+}=0$, and $C$ is positive constant. 
Note that, nnBD-LSIF has an unbiased gradient. 
The optimization efficiency of nnBD-LSIF was observed to confirm the effectiveness of an unbiased gradient of an $f$-divergence loss function as well as $\alpha$-Div.

\subsubsection{Neural Network Architecture, Optimization Algorithm, and Hyperparameters.} 
We used a 4-layer perceptron with ReLU activation, in which each hidden layer contained 100 nodes. For optimization, we set the learning rate to 0.00005, the batch size to 2500, and the number of epochs to 1000. 
We implemented all models for DRE using the PyTorch library \citep{paszke2017automatic} in Python. Training was performed with the Adam optimizer \citep{kingma2014adam} in PyTorch, utilizing an NVIDIA T4 GPU. 

\subsection{Details of the experiments in Section \ref{subsection_ExperimentsontheestimationaccuracyusinghighKLdivergencedata}} 
\label{Apdx_subsection_ExperimentsontheestimationaccuracyusinghighKLdivergencedata} 
In this section, we provide details of the experiments reported in Section \ref{subsection_ExperimentsontheestimationaccuracyusinghighKLdivergencedata}. 

\subsubsection{Datasets.} 
Initially, we created 100 train and test datasets, each with a size of 10,000. 
Each dataset is generated from two normal distributions $P=\mathcal{N}(\mu_p, \sigma^2 \cdot I_3)$ and $Q=\mathcal{N}(\mu_q,  4^2 \cdot I_3)$ where $I_3$ denotes the $3$-dimensional identity matrix and $\sigma$ values were $1.0$, 1.1, 1.2, 1.4, 1.6, 2.0, 2.5, or 3.0, and $\mu_p=(-3/2, -3/2, -3/2)$ and $\mu_q=(3/2,3/2,3/2)$. 
In the aforementioned setting, the ground truth KL-divergence amounts of the datasets is obtained as 
\begin{align} 
    KL(P||Q) &= E_P \left[ \log \left(\frac{dP}{dQ}   \right)  \right] \nonumber\\ 
    &= \frac{1}{2} \cdot \left[ \log \frac{|\Sigma_p|}{|\Sigma_q|} -d + \mathrm{Tr }(\Sigma_p^{-1} \cdot \Sigma_q) + (\mu_p - \mu_q)^{T} \cdot \Sigma_p^{-1} \cdot  (\mu_p - \mu_q) \right] \nonumber\\ 
    &=  \frac{1}{2} \cdot  \left[ \log \frac{\sigma^2 \cdot|I_3|}{4^2\cdot|I_3|} - 3 + \mathrm{Tr}(\sigma^{-2} \cdot I_3\cdot 4^2\cdot I_3) + 3 \cdot \mathbf{1}^{T}  \cdot \sigma^{-2} \cdot I_3 \cdot   3 \cdot \mathbf{1} \right] \nonumber\\ 
     &=  \frac{1}{2} \cdot  \Big(6 \log \sigma - 12\log 2  - 3+  3 \cdot \sigma^{-2}\cdot 16 + 27 \cdot \sigma^{-2} \Big)  \nonumber\\ 
      &=  3 \log \sigma - 6 \log 2  - \frac{3}{2}+   \sigma^{-2}\cdot \frac{75}{2}. 
    \label{Eq_apedx_Apdx_subsection_ExperimentsontheestimationaccuracyusinghighKLdivergencedata} 
\end{align} 
From Equation (\ref{Eq_apedx_Apdx_subsection_ExperimentsontheestimationaccuracyusinghighKLdivergencedata}), we see that the ground truth KL-divergence amounts of the datasets were 31.8, 25.6, 21.0, 14.5, 10.4, 10.4, 5.8, 3.1, and 1.8, which correspond to the ascending $\sigma$ values, such that $\sigma=1.0, 1.1, \ldots, 3.0$. 

\subsubsection{Experimental Procedure.} 
We trained neural networks using the training datasets by optimizing both $\alpha$-Div with $\alpha=0.5$ and a KL-divergence loss function. 
Details of the KL-divergence loss function used in the experiments are provided in the following paragraph. 
Training was halted if the validation losses, measured using the validation datasets, did not improve during an entire epoch. 
After training the neural networks, we measured the root mean squared error (RMSE) of the estimated density ratios using the test datasets. 
We estimated the KL-divergence of the test datasets for each trial using the estimated density ratios and the plug-in estimation method, which is detailed below. 
A total of 100 trials were conducted. Finally, we reported the median RMSE of the DRE and the estimated KL-divergence, along with the interquartile range (25th to 75th percentiles), for each KL-divergence loss function and $\alpha$-Div. 

\paragraph{KL-divergence loss function.} 
A standard KL-divergence loss function is obtained as 
\begin{equation} 
    \mathcal{L}_{\mathrm{standard}\text{-}\mathrm{KL}} (\phi) = 
    \hat{E}_P\left[\phi\right] - \hat{E}_Q\left[\log \phi \right]. 
\end{equation} 
In our pre-experiment, the standard KL-divergence loss function exhibited poor optimization performance, which we attributed to its biased gradients. 
However, we found that applying the Gibbs density transformation, as described in Section \ref{DerivationofourlossfunctionforDRE}, improved optimization performance for the KL-divergence loss function. Therefore, we used the following KL-divergence loss function, $\mathcal{L}_{\mathrm{KL}}(\cdot)$ in our experiments: 
\begin{equation} 
    \mathcal{L}_{\mathrm{KL}}(T) = 
     \hat{E}_P\left[e^{T} \right] - \hat{E}_Q\left[T \right]. 
\end{equation} 

\paragraph{Plug-in KL-divergence estimation method using the estimated density ratios.} 
The KL-divergence of the test datasets was estimated by estimated predicted density ratios for the test datasets using  plug-in estimation, such that 
\begin{equation} 
    \widehat{KL}(P||Q) = \hat{E}_Q \left[\log \hat{r}_q(\mathbf{x}) \right], 
    \label{eq_apendix_KL_div_loss_in_exp} 
\end{equation} 
where $\hat{r}_q(\mathbf{x}) = e^{T(\mathbf{x})} / \hat{E}_Q[ e^{T(\mathbf{x})}]$. 

\subsubsection{Neural Network Architecture, Optimization Algorithm, and Hyperparameters.} 
The same neural network architecture, optimization algorithm, and hyperparameters were used for both $\alpha$-Div and the KL-divergence loss function. 
A 4-layer perceptron with ReLU activation was employed, with each hidden layer consisting of 256 nodes. For optimization with the $\alpha$-Div loss function, the value of $\alpha$ was set to 0.5, the learning rate to 0.00005, and the batch size to 256. Early stopping was applied with a patience of 32 epochs, and the maximum number of epochs was set to 5000. For optimization using the KL-divergence loss function, the learning rate was 0.00001, with a batch size of 256. Early stopping was applied with a patience of 2 epochs, and the maximum number of epochs was 5000. 
All models for both D3RE and $\alpha$-Div were implemented using PyTorch library \citep{paszke2017automatic} in Python. The neural networks were trained with the Adam optimizer \citep{kingma2014adam} in PyTorch on an NVIDIA T4 GPU.

\section{Additional Experiments}\label{Section_Appendix_Additionalexperiments} 

\subsection{Comparison with Existing DRE Methods}\label{Apdx_subsection_Additionalexperiments_ExperimentsonComparisontoExistingDREmethod} 
We empirically compared the proposed DRE method with existing DRE methods in terms of DRE task accuracy. 
This experiment followed the setup described in \citet{kato2021non}. 

\subsubsection{Existing \texorpdfstring{$f$}{f}-Divergence Loss Functions for Comparison.} 
The proposed method was compared with the Kullback--Leibler importance estimation procedure (KLIEP) \citep{sugiyama2007direct}, unconstrained least-squares importance fitting (uLSIF) \citep{kanamori2009least}, and deep direct DRE (D3RE) \citep{kato2021non}. 
The \texttt{densratio} library in R was used for KLIEP and uLSIF.
\footnote{The URL: \texttt{https://cran.r-project.org/web/packages/densratio/index.html}.} 
For D3RE, the non-negative Bregman divergence least-squares importance fitting (nnBD-LSIF) loss function was employed. 

\subsubsection{Datasets.} 
For each $d=10, 20, 30, 50$, and $100$, 100 datasets were generated, comprising training and test sets drawn from two $d$-dimensional normal distributions $P = \mathcal{N}(\mu_p, I_d)$ and $Q = \mathcal{N}(\mu_q, I_d)$, where $I_d$ denotes the $d$-dimensional identity matrix, $\mu_p = (0, 0, \dots, 0)$, and $\mu_q = (1, 0, \dots, 0)$. 

\subsubsection{Experimental Procedure.} 
Model parameters were trained using the training datasets, and density ratios for the test datasets were estimated. 
The mean squared error (MSE) of the estimated density ratios for the test datasets was calculated based on the true density ratios. 
Finally, the mean and standard deviation of the MSE for each method were reported. 

\subsubsection{Neural Network Architecture, Optimization Algorithm, and Hyperparameters.} 
For both D3RE and $\alpha$-Div, a 3-layer perceptron with 100 hidden units per layer was used, consistent with the neural network structure employed in \citet{kato2021non}. 
For D3RE, the learning rate was set to 0.00005, the batch size to 128, and the number of epochs to 250 for each data dimension. 
The hyperparameter $C$ was set to 2.0. For $\alpha$-Div, the learning rate was set to 0.0001, the batch size to 128, and the value of $\alpha$ to 0.5 for each data dimension. 
The number of epochs was set to 40 for data dimensions of 10, 50 for dimensions of 20, 30, and 50, and 60 for a dimension of 100. 
The PyTorch library \citep{paszke2017automatic} in Python was used to implement all models for both D3RE and $\alpha$-Div. 
The Adam optimizer \citep{kingma2014adam} in PyTorch, along with an NVIDIA T4 GPU, was used for training the neural networks.

\paragraph{Results.} 
\begin{table*}[t] 
    \caption{
    Results of additional experiments described in Section 
    \ref{Apdx_subsection_Additionalexperiments_ExperimentsonComparisontoExistingDREmethod}. 
    The table reports the mean and standard deviation of the MSE for DRE with each method. 
    Results are presented in the format ``mean (standard deviation)''. The lowest MSE values are highlighted in bold.} 
    \label{Table_for_Experiment_DRE} 
    \centering 
     \vspace{3.0mm} 
    \begin{tabular}{lccccc} 
        \toprule            %
        & 
        \multicolumn{5}{c}{Data dimensions ($d$)} \\ 
        \cmidrule(lr){2-6} 
        Model 	& $d=10$ & $d=20$	& $d=30$  & $d=50$ & $d=100$ \\ 
        \midrule 
        KLIEP & 2.141(0.392) & 2.072(0.660) & 2.005(0.569) & 1.887(0.450) & 1.797(0.419) \\ 
        uLSIF & 1.482(0.381) & 1.590(0.562) & 1.655(0.578) & 1.715(0.446) & 1.668(0.420) \\ 
        D3RE  & 1.111(0.314) & 1.127(0.413) & 1.219(0.458) & 1.222(0.305) & 1.369(0.355) \\ 
        $\alpha$-Div    &  \textbf{0.173(0.072)} &  \textbf{0.278(0.113)} &  \textbf{0.479(0.259)} & \textbf{0.665(0.194)} & \textbf{1.118(0.314)} \\ 
        \bottomrule 
    \end{tabular} 
\end{table*} 

Table \ref{Table_for_Experiment_DRE} summarizes the results for each method across different data dimensions. 
Six cases where the MSE for KLIEP exceeded 1000 were excluded. 
For all data dimensions, $\alpha$-Div consistently demonstrated superior accuracy compared to the other methods, achieving the lowest MSE values. However, it is important to note that the prediction accuracy of $\alpha$-Div significantly decreased as the data dimensions increased. 
The curse of dimensionality in DRE was also observed in experiments with real-world data, which will be reported in the next section.

\subsection{Experiments Using Real-World Data} 
\label{Apdx_subsection_AdditionalexperimentsExperimentsUsingRealWorldData} 
We conducted  numerical experiments with real-world data to highlight important considerations in applying the proposed method. Specifically, we conducted experiments on Domain Adaptation (DA) for classification models using the Importance Weighting (IW) method \citep{shimodaira2000improving}. 
The IW method builds a prediction model for a target domain using data from a source domain, while adjusting the distribution of source domain features to match the target domain features by employing the density ratio between the source and target domains as sample weights. 

In these experiments, we used the Amazon review data \citep{blitzer2007biographies} and employed two prediction algorithms: linear regression and gradient boosting. The hyperparameters for each algorithm were selected from a predefined set based on validation accuracy, using the Importance Weighted Cross Validation (IWCV) method \citep{sugiyama2007covariate} on the source domain data. 

Through these experiments, we observed a decline in prediction accuracy on test data from the target domain as the data dimensionality increased. 
Specifically, there were instances where the accuracy worsened compared to models that did not use importance weighting---i.e., models trained solely on the source data. 
These phenomena are likely due to two issues in DRE: the degradation in density ratio estimation accuracy as dimensionality increases, as noted in Section \ref{Apdx_subsection_Additionalexperiments_ExperimentsonComparisontoExistingDREmethod}, and the negative impact of high KL-divergence on density ratio estimation, as observed in Section \ref{subsection_ExperimentsontheestimationaccuracyusinghighKLdivergencedata}. 
It is important to note that the KL-divergence increases as the number of features increases (i.e., data dimensions), unless all features are fully independent.

\subsubsection{Datasets.} 
The Amazon review dataset \citep{blitzer2007biographies} includes text reviews and rating scores from four domains: books, DVDs, electronics, and kitchen appliances. The text reviews are one-hot encoded, and the rating scores are converted into binary labels. Twelve domain adaptation classification tasks were conducted, where each domain served once as the source domain and once as the target domain.

\paragraph{Notation.} 
$\mathbf{X}_{\text{S}}^{d}$ and $\mathbf{X}_{\text{T}}^{d}$ denote subsets of the original data for the source and target domains, respectively, for each feature dimension $d$, where the columns of $\mathbf{X}_{\text{S}}^{d}$ and $\mathbf{X}_{\text{T}}^{d}$ are identical. $y_{\text{S}}$ and $y_{\text{T}}$ represent the objective variables in the source and target domains, respectively, which are binary labels assigned to each sample in the source and target domain data. $\mathbf{Z}_{\text{S}}^d$ and $\mathbf{Z}_{\text{T}}^d$ denote $d$-dimensional feature tables used to estimate the density ratio $\hat{r}(\mathbf{Z}_{\text{S}}^{d})$, which is the ratio of the target domain density to the source domain density. $\text{dim}(X)$ indicates the number of columns (features) in the data $X$.

\subsubsection{Experimental Procedure.} 
\paragraph{Step 1. Creation of feature tables.} 
Many DA methods utilize feature embedding techniques to project high-dimensional data into a lower-dimensional feature space, facilitating the handling of distribution shifts between source and target domains \citep{ragab2023adatime}. However, our preliminary experiments revealed that model prediction accuracies were significantly influenced by the embedding procedures. To address these effects on DA task accuracies, we explored an embedding method with theoretical considerations detailed in the next section. 

Specifically, we selected an identical set of columns from the original data of both the source and target domains for each feature dimension, $d = 8, 16, 32, 64$, and $128$, arranging the columns in ascending order of $d$. Let $\mathbf{X}_{\text{S}}^{d}$ and $\mathbf{X}_{\text{T}}^{d}$ denote the subsets of the original data for the source and target domains, respectively, determined by these selected columns for each $d$. We then generated a $d \times \text{dim}(\mathbf{X}_{\text{S}}^d)$ matrix $A_d$ from a normal distribution. Finally, by multiplying $\mathbf{X}_{\text{S}}^{d}$ and $A_d$, and $\mathbf{X}_{\text{T}}^{d}$ and $A_d$, we obtained the feature tables $\mathbf{Z}_{\text{S}}^d$ and $\mathbf{Z}_{\text{T}}^d$, embedding the original source and target domain data into a $d$-dimensional feature space. 
\footnote{In our experiments, we utilized matrices generated from the normal distribution as 
    embedding maps, which is equivalent to random projection \citep{bingham2001random}. 
    However, the linearity of the map is not necessary for preserving the density ratios, as discussed in Section \ref{subsection_ConsiderationoftheFeatureEmbeddingMethod}. 
    In contrast, linearity is a key requirement for the distance-preserving property of random projection.} 

\paragraph{Step 2. Estimation of importance weights.} Using the proposed loss function with $\mathbf{Z}_{\text{S}}^{d}$ and $\mathbf{Z}_{\text{T}}^{d}$ obtained from the previous step, we estimated the probability density ratio $\hat{r}(\mathbf{Z}_{\text{S}}^{d})$ for each feature dimension $d$, where $r(\mathbf{Z}_{\text{S}}^d) = {q(\mathbf{Z}_{\text{T}}^d)}/{p(\mathbf{Z}_{\text{S}}^d)}$. This ratio represents the density of the target domain relative to the source domain.

\paragraph{Step 3. Model construction.} 
We constructed the target model using the IW method. Specifically, we built a classification model using the training dataset $(\mathbf{X}_{\text{S}}^d, y_{\text{S}})$, where the estimated density ratio $\hat{r}(\mathbf{Z}_{\text{S}}^d)$ served as the sample weights for the IW method. Additionally, we constructed a prediction model using only the source data, i.e., a model built without importance weighting.

\paragraph{Step 4. Verification of prediction accuracy} 
To evaluate the prediction accuracy of the models, we selected the ROC AUC score, as it measures the accuracy independent of the thresholds used for label determination. 
For the classification tasks in domain adaptation, we employed two classification methods, each representing a different algorithmic approach: \texttt{LogisticRegression} from the \texttt{scikit-learn} library \citep{pedregosa2011scikit} for linear classification, and \texttt{LightGBM} \citep{ke2017lightgbm} for nonlinear classification. 

The hyperparameter sets of both methods for evaluating the prediction accuracies on the target domains were selected using the IWCV method \citep{sugiyama2007covariate}. These hyperparameter sets were defined as all combinations of the values listed in Table \ref{Table_HyperparametervaluesforElasticLogisticRegression} for \texttt{LogisticRegression} and Table \ref{Table_HyperparametervaluesforElasticLightGBM} for \texttt{LightGBM}, respectively. 
Finally, the prediction accuracies on the target domain were assessed using the best model selected through IWCV, where the target domain data $(\mathbf{X}{\text{T}}^d, y{\text{T}})$ were used for the predictions.

\renewcommand{\arraystretch}{1.3} 
\begin{table}[ht] 
    \caption{Hyperparameter values for LogisticRegression. ``Hyperparameters'' shows the hyperparameter names used in the library. Texts inside parentheses provide explanations of the parameters.} 
    \label{Table_HyperparametervaluesforElasticLogisticRegression} 
    \centering 
        \vspace{3.0mm} 
    \begin{tabularx}{\textwidth}{p{6cm}>{\RaggedRight}X} 
        \toprule 
        Hyperparameters & Values \\ 
        \hline 
        l1\_ratio (Elastic-Net mixing parameter) & 0, 0.1, 0.2, 0.3, 0.4, 0.5, 0.6, 0.7, 0.8, 0.9, and 1.0 \\ 
        lambda (Inverse of regularization strength) & 0.0001, 0.001, 0.01, 0.05, 0.1, 0.25, 0.5, 0.75, 1, 1.5, 2, and 5 \\ 
        \hline 
    \end{tabularx} 
\end{table} 
\renewcommand{\arraystretch}{1.0} 
 \vspace{4.2mm} 

\renewcommand{\arraystretch}{1.3} 
\begin{table}[ht] 
    \caption{Hyperparameter values for LightGBM. ``Hyperparameters'' shows the hyperparameter names used in the library. Texts inside parentheses provide explanations of the parameters.} 
    \label{Table_HyperparametervaluesforElasticLightGBM} 
    \centering 
    \vspace{3.0mm} 
    \begin{tabularx}{\textwidth}{p{8cm}>{\RaggedRight}X} 
        \toprule 
        Hyperparameters & Values \\ 
        \hline 
        lambda\_l1 ($L_1$ regularization) & 0.0, 0.25, 0.5, 0.75, and 1.0 \\ 
        lambda\_l2 ($L_2$ regularization) & 0.0, 0.0001, 0.001, 0.01, 0.1, 0.5, 1.0, 2.0, and 4.0\\ 
        num\_leaves (Number of leaves in trees) & 64, 248, 1024, 2048, and 4096 \\ 
        learning\_rate (Learning rate) & 0.01, and 0.001 \\ 
        feature\_fraction (Ratio of featuresr used for modeling) & 0.4, 0.8, and 1.0 \\ 
        \hline 
    \end{tabularx} 
 \vspace{4.2mm} 
\end{table} 
\renewcommand{\arraystretch}{1.0} 

\renewcommand{\arraystretch}{1.3} 
\begin{table}[ht] 
    \caption{ 
        Original data dimensions ($\text{dim}(\mathbf{X})$) used to obtain feature dimensions ($\text{dim}(\mathbf{Z})$) by embedding.} 
    \label{table_Dimensionsoforiginaldataandcorrespondingfeaturedimensions} 
    \centering 
    \vspace{3.0mm} 
    \begin{tabular}{cccccc} 
        \hline 
        & \multicolumn{5}{c}{Feature dimensions ($d = \text{dim}(\mathbf{Z})$)} \\ 
        \cline{2-6} 
        & \( d = 8 \) & \( d = 16 \) & \( d = 32 \) & \( d = 64 \) & \( d = 128 \) \\ 
        \hline 
        Original data dimensions ($\text{dim}(\mathbf{X})$) &500 & 700 & 900 & 1700 & 4600 \\ 
        \hline 
    \end{tabular} 
\end{table} 
 \vspace{4.2mm} 
\renewcommand{\arraystretch}{1.0}

\subsubsection{Consideration of the Feature Embedding Method}\label{subsection_ConsiderationoftheFeatureEmbeddingMethod} 
Let $f: \mathbf{X} \longmapsto \mathbf{Z}$ denote a $C^1$-class embedding map 
which maps the original data $\mathbf{X} \subseteq \mathbb{R}^{N \times D}$ into the feature space $\mathbf{Z} \subseteq \mathbb{R}^{N \times d}$ with $d < D$. 

We now demonstrate that if $f$ is injective for both the source and target domain data, it preserves the density ratio between the target and source domain densities when it maps the original data into the embedded data. 

To demonstrate this, we use the singular value decomposition (SVD) of the Jacobian matrix $J_f(\mathbf{x})$ of $f$, which gives 

\[ 
J_f(\mathbf{x}) = U(\mathbf{x}) \cdot \Sigma(\mathbf{x}) \cdot V^T(\mathbf{x}), 
\] 

with 

\[ 
\Sigma(\mathbf{x}) = \begin{pmatrix} 
    \sigma_1(\mathbf{x}) & 0 & \dots & 0 \\ 
    0 & \sigma_2(\mathbf{x}) & \dots & 0 \\ 
    \vdots & \vdots & \ddots & \vdots \\ 
    0 & 0 & \dots & \sigma_{\text{dim}(\mathbf{z})}(\mathbf{x}) \\ 
    0 & 0 & \dots & 0 \\ 
    \vdots & \vdots & \dots & \vdots \\ 
    0 & 0 & \dots & 0 \\ 
\end{pmatrix}, 
\] 

where $U(\mathbf{x})$ and $V^T(\mathbf{x})$ are orthogonal matrices in $\mathbb{R}^{\text{dim}(\mathbf{x}) \times \text{dim}(\mathbf{x})}$ and $\mathbb{R}^{\text{dim}(\mathbf{z}) \times \text{dim}(\mathbf{z})}$, respectively, and $\sigma_i(\mathbf{x}) \neq 0$ for all $i$. 
This gives the following relationship between the probability densities of the original and embedded data:

\begin{equation} 
    p_{\mathbf{X}}(\mathbf{x}) = \left( \prod_{i=1}^{\text{dim}(\mathbf{z})} \sigma_i(\mathbf{x}) \right) p_{\mathbf{Z}}(f(\mathbf{x})). \label{Eq_seqtion_D_2_1} 
\end{equation} 

From Equation (\ref{Eq_seqtion_D_2_1}), the probability density ratio between the source and target domains of data embedded by $f$ is obtained as 

\[ 
\frac{q_{\mathbf{X}}(\mathbf{x})}{p_{\mathbf{X}}(\mathbf{x})} = \frac{\left( \prod_{i=1}^{\text{dim}(\mathbf{z})} \sigma_i(\mathbf{x}) \right) \cdot q_{\mathbf{Z}}(f(\mathbf{x}))}{\left( \prod_{i=1}^{\text{dim}(\mathbf{z})} \sigma_i(\mathbf{x}) \right) \cdot p_{\mathbf{Z}}(f(\mathbf{x}))} = \frac{q_{\mathbf{Z}}(\mathbf{z})}{p_{\mathbf{Z}}(\mathbf{z})}. 
\] 

Therefore, $f$ preserves the density ratio from the original data to the embedded data. 
 Additionally, if $f$ is a matrix multiplication, its injectivity can be achieved for $\mathbf{X}$ by reducing its dimensionality sufficiently. 
Reducing the dimensionality of $\mathbf{X}$ can ensure the injectivity of $f$.

We heuristically detected the injectivity of our embedding by observing the following: 
We identified the largest subset of columns in 
 $\mathbf{Z}_{\text{S}}^d$ such that a significant increase in the KL-divergence between $P(\mathbf{Z}_{\text{S}}^d)$ and $P(\mathbf{Z}_{\text{T}}^d)$ was observed when a column is added to a partial subset of columns within it. Injectivity was assumed for columns within this subset. 

Although our feature embedding procedure is based on heuristic observations and lacks rigorous theoretical analysis, we found it adequate for evaluating the performance of the proposed method in DRE downstream tasks with real-world data when the number of features increases. 

 The number of columns in the original data used in the experiments is listed in Table \ref{table_Dimensionsoforiginaldataandcorrespondingfeaturedimensions}.

\paragraph{Neural Network Architecture, Optimization Algorithm, and Hyperparameters.} 
A 5-layer perceptron with ReLU activation was used, with each hidden layer consisting of 256 nodes. For optimization, the value of $\alpha$ was set to 0.5, the learning rate to 0.0001, and the batch size to 128. Early stopping was applied with a patience of 1 epoch, and the maximum number of epochs was set to 5000. 
The PyTorch library \citep{paszke2017automatic} in Python served as the framework for model implementation. Training of the neural networks was carried out using the Adam optimizer \citep{kingma2014adam} on an NVIDIA T4 GPU.

\paragraph{Results.} The results are shown in Figure \ref{Apdx_subsection_FigLogisticRegression} (LogisticRegression) and Figure \ref{Apdx_subsection_FigLightGBM} (LightGBM). The domain names at the origin of the arrows in the figure titles represent the source domains, and those at the tip indicate the target domains. The $x$-axis of each figure shows the number of features, and the $y$-axis represents the ROC AUC for the domain adaptation tasks. 
The orange line (SO) represents models trained using source-only data, i.e., models trained using source data without importance weighting, while the blue line (IW) represents models trained using source data with importance weighting. 

Prediction accuracy for the models trained solely on the source data improved as the number of features increased, which is expected since more features typically lead to better accuracy. However, for both Logistic Regression and LightGBM, the performance of the IW method deteriorated as the number of features increased. A more significant decline in performance with increasing features was observed for most domain adaptation (DA) tasks, except for ``books $\rightarrow$ DVDs'' and ``kitchen $\rightarrow$ DVDs''. 
These results suggest that the estimated density ratios caused the data distribution to deviate more from the target domain as the number of features increased. Consequently, the accuracy of the density ratio estimation (DRE) likely worsened with more features.

\begin{figure*}[t] 
    \begin{center} \centerline{\includegraphics[width=1.00\columnwidth]{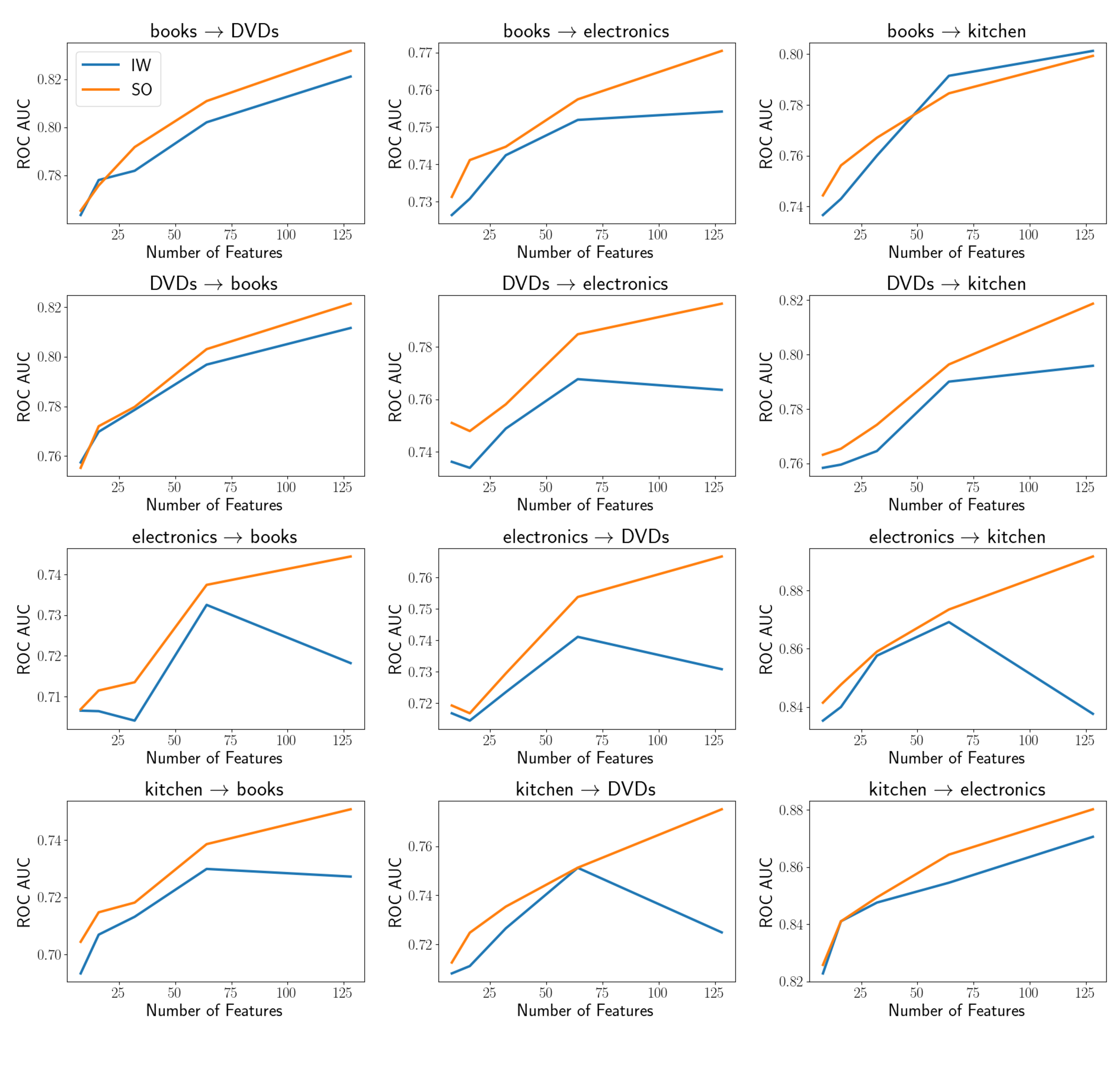}} 
        \caption{Results of Section \ref{Apdx_subsection_AdditionalexperimentsExperimentsUsingRealWorldData} for 
      LogisticRegression. In the figure titles, domain names at the origin of the arrows indicate the source domains, while those at the tip represent the target domains. The $x$-axis shows the number of features, and the $y$-axis represents the ROC AUC for the domain adaptation tasks. 
      The orange line (SO) denotes models trained using source-only data (i.e., models trained on source data only, without importance weighting), whereas the blue line (IW) represents models trained using source data with importance weighting.} 
        \label{Apdx_subsection_FigLogisticRegression} 
    \end{center} 
\end{figure*} 
\begin{figure*}[t] 
    \begin{center} \centerline{\includegraphics[width=1.00\columnwidth]{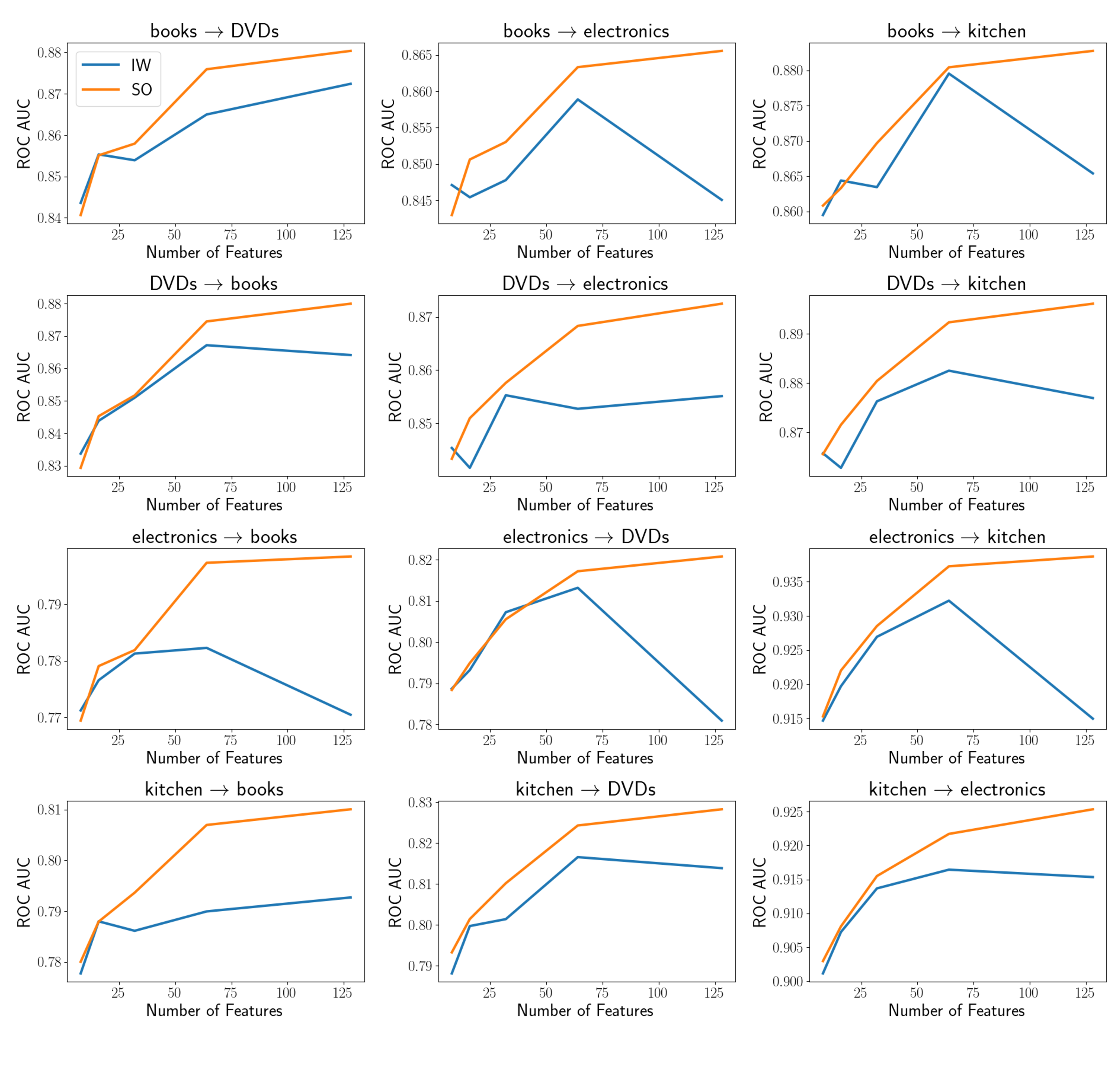}} \caption{Results of Section \ref{Apdx_subsection_AdditionalexperimentsExperimentsUsingRealWorldData}. 
            \label{Apdx_subsection_FigLightGBM} 
    Results of Section \ref{Apdx_subsection_AdditionalexperimentsExperimentsUsingRealWorldData} for LightGBM. In the figure titles, domain names at the origin of the arrows represent the source domains, while those at the tip indicate the target domains. The $x$-axis shows the number of features, and the $y$-axis represents the ROC AUC for the domain adaptation tasks. The orange line (SO) denotes models trained using source-only data (i.e., models trained on source data only, without importance weighting), whereas the blue line (IW) represents models trained using source data with importance weighting.} 
    \end{center} 
\end{figure*} 


\end{document}